\documentclass[10pt,journal,compsoc]{IEEEtran}
\ifCLASSOPTIONcompsoc
  \usepackage[nocompress]{cite}
\else
  \usepackage{cite}
\fi

\ifCLASSINFOpdf
\else
\fi

\usepackage[pdftex]{graphicx}
\DeclareGraphicsExtensions{.pdf,.jpeg,.png}
\usepackage{amsmath}
\usepackage{amsthm}
\usepackage{algorithmic}
\usepackage{array}
\usepackage{fixltx2e}
\usepackage{stfloats}
\usepackage{algorithm}
\usepackage{url}
\usepackage{cite}
\newtheorem{prop}{Proposition}

\usepackage{todonotes}

\def\ie{\textit{i}.\textit{e}. }
\def\eg{\textit{e}.\textit{g}. }
\def\wrt{\textit{w}.\textit{r}.\textit{t}.}
\def\etal{\textit{et al}.}
\def\etc{etc.}
\def\VEC#1{\mbox{\boldmath $#1$}}
\def\PAIR#1#2{\langle #1, #2 \rangle}

\begin{document}

\title{
  Structure of Multiple Mirror System from Kaleidoscopic Projections of Single
  3D Point}
%

%
%

\author{Kosuke~Takahashi,~\IEEEmembership{Non~Member,~IEEE,}
  and~Shohei~Nobuhara,~\IEEEmembership{Member,~IEEE}
  \IEEEcompsocitemizethanks{\IEEEcompsocthanksitem K. Takahashi and S. Nobuhara
    are with the Graduate School of Informatics Kyoto University, Kyoto,
    Japan.\protect\\
E-mail: t.kosuke.1210@gmail.com, nob@i.kyoto-u.ac.jp}
\thanks{\copyright2021 IEEE.  Personal use of this material is permitted.  Permission from IEEE must be obtained for all other uses, in any current or future media, including reprinting/republishing this material for advertising or promotional purposes, creating new collective works, for resale or redistribution to servers or lists, or reuse of any copyrighted component of this work in other works.}
}

%
%

\markboth{}{}
%


\IEEEtitleabstractindextext{%
\begin{abstract}
  This paper proposes a novel algorithm of discovering the structure of a kaleidoscopic imaging system that consists of multiple planar mirrors and a camera. The kaleidoscopic imaging system can be recognized as the virtual multi-camera system and has strong advantages in that the virtual cameras are strictly synchronized and have the same intrinsic parameters. In this paper, we focus on the extrinsic calibration of the virtual multi-camera system. The problems to be solved in this paper are two-fold. The first problem is to identify to which mirror chamber each of the 2D projections of mirrored 3D points belongs. The second problem is to estimate all mirror parameters, \ie, normals, and distances of the mirrors. The key contribution of this paper is to propose novel algorithms for these problems using a single 3D point of unknown geometry by utilizing a {\it kaleidoscopic projection constraint}, which is an epipolar constraint on mirror reflections. We demonstrate the performance of the proposed algorithm of chamber assignment and estimation of mirror parameters with qualitative and quantitative evaluations using synthesized and real data.
\end{abstract}

}


\maketitle

\IEEEdisplaynontitleabstractindextext


\IEEEpeerreviewmaketitle

\IEEEraisesectionheading{\section{Introduction}\label{sec:section_1}}
\IEEEPARstart{A} kaleidoscopic imaging system consisting of multiple planar mirrors is a practical approach to realize a multi-view capture of a target by synchronized cameras with identical intrinsic parameters (Figure \ref{fig:multi_mirror_system}). It can be recognized as a virtual multiple-view system and has been widely used for 3D shape reconstruction by stereo\cite{nane98stereo,goshtasby93design,gluckman2001catadioptric}, shape-from-silhouette\cite{huang06contour,forbes06shape,reshetouski11three}, structured-lighting\cite{lanman09surround,tahara15interference}, and time of flight (ToF)\cite{nobuhara16single} and also for reflectance analysis\cite{ihrke2012kaleidoscopic,tagawa12eight,inoshita13full}, light-field imaging\cite{levoy04synthetic,sen05dual,fuchs12design}, etc. Such applications require extrinsic calibration of the virtual cameras, \ie, estimation of their poses and positions. Since calibrating the virtual cameras is identical to calibrating the mirror normals and distances, this paper proposes a novel algorithm of extrinsic calibration of the kaleidoscopic imaging system by estimating the mirror parameters from a kaleidoscopic projection of a single 3D point.

To calibrate the kaleidoscopic imaging system, two problems need to be solved: {\it chamber assignment} and {\it mirror parameters estimation}. The first problem is to identify to which \textit{chamber} (\ie, the image region corresponding to the direct view, the first reflection, the second reflection, and so forth in the captured image) each 2D projection of mirrored 3D points belongs. The second problem is to estimate the normals and the distances of all mirror planes from the 2D projections in chambers.

The key contribution of this paper is to propose novel algorithms for solving these two problems by introducing an epipolar constraint on mirror reflections that is satisfied by projections of high-order reflections \cite{gluckman2001catadioptric, ying10geometric, mariottini2010catadioptric}, which is called a {\it kaleidoscopic projection constraint} in this paper. This constraint provides multiple linear equations on mirror parameters from a single 3D point and enables the proposed method to conduct a 3D validation on candidates of chamber assignments and to estimate mirror parameters linearly. Note that our method assumes that the kaleidoscopic imaging system satisfies the following conditions,
\begin{itemize}
\item no pair of mirrors is parallel,
\item the second reflections of the single 3D point are observable.
\end{itemize}

\begin{figure}[t]
\centering
\includegraphics[width=0.9\linewidth]{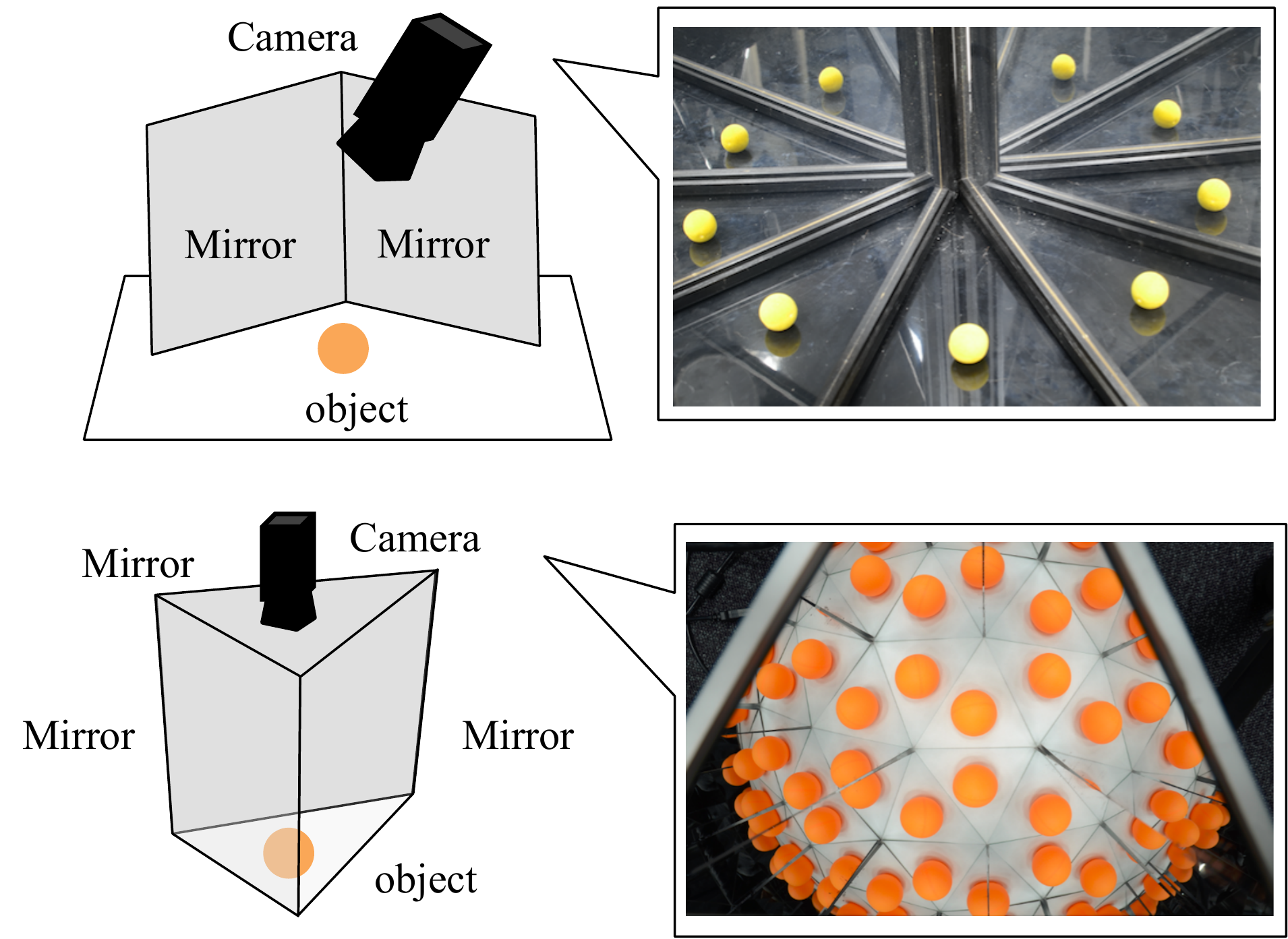}
\caption{Kaleidoscopic imaging system consists of multiple planer mirrors. Up: two mirrors. Bottom: three mirrors. The projections of the reflected target object in each chamber are recognized as the target observed by the virtual cameras, which are generated by the planer mirrors.}
\label{fig:multi_mirror_system}
\end{figure}

An earlier version of this study was presented in \cite{takahashi17linear}. Compared with \cite{takahashi17linear}, this paper provides a novel algorithm of chamber assignment based on the kaleidoscopic projection constraint, while \cite{takahashi17linear} assumes that the chamber labels are assigned by hand beforehand. This chamber assignment algorithm realizes a fully automatic calibration of the kaleidoscopic imaging system. In addition, this version updates the kaleidoscopic bundle adjustment in the calibration step and conducted exhaustive evaluations with conventional approaches.

The rest of this paper is organized as follows. Section \ref{sec:section_2} reviews related studies to clarify our contribution. Section \ref{sec:section_3} introduces the problem definition in our kaleidoscopic imaging system and the key constraint; kaleidoscopic projection constraint. Sections \ref{sec:section_5} and \ref{sec:section_6} propose our chamber assignment algorithm and mirror parameters estimation algorithm, respectively. Section \ref{sec:section_7} evaluates the proposed method quantitatively and qualitatively, and section \ref{sec:section_8} discusses the limitations and possible extensions. Section \ref{sec:section_9} concludes the paper and outlines future works.

\section{Related works}\label{sec:section_2}
Conventional studies utilizing mirrors can be categorized
into two groups: (1) studies utilizing mirrors as supplemental
devices for calibrating camera systems, and (2) studies integrating mirrors as components into their imaging systems.

\subsection{Mirrors as supplemental devices for calibration}\label{sec:related_work_1} The first group utilizes mirrors for extrinsic calibration of a camera and a reference object which is located out of its field-of-view and can be categorized into two subgroups: with planar mirrors and with non-planar mirrors.

Sturm \etal \cite{sturm06how} and Rodrigues \etal \cite{rodorigues10camera} propose a planar-mirror based method of computing the pose of a reference object without a direct view. Hesch \etal\cite{hesch08mirror, Hesch2010extrinsic} assume a camera-based robot in which no body parts are visible from the camera and propose a linear method of estimating their relative pose and positions with a planar mirror in a perspective-three-point (P3P) problem scenario. Takahashi \etal \cite{takahashi12new} propose a planar mirror-based method with a minimal configuration using three mirror poses and three reference points based on orthogonality constraints. Delaunoy \etal \cite{delaunoy2014two} calibrate a mobile device that has one display and two cameras on its front and back side by using planar mirror reflections and reconstructed scenes. Kumar \etal \cite{kumar08simple} and Bunshnevskiy \etal \cite{Bushnevskiy_2016_CVPR} calibrate a camera network and omnidirectional camera consisting of multiple perspective cameras with planar mirrors.

Agrawal \cite{agrawal13extrinsic} calibrates by utilizing a coplanarity constraint satisfied by the reflections of eight reference points with a single spherical mirror pose. Instead of using a spherical mirror, Nitschke \etal \cite{nitschke2011display} and Takahashi \etal \cite{takahashi2016extrinsic} recognize the human eyeball as a spherical mirror and proposed a cornea-reflection based calibration method in a display-camera system.

\subsection{Mirrors as imaging components}
The second group utilizes non-planar or planar mirrors as a component of their imaging systems. 

Non-planar mirrors are commonly employed for widening field-of-view. One of the most common usages of non-planar mirrors in an imaging system is for capturing omnidirectional images. Scaramuzza \etal\cite{scaramuzza2006flexible} proposes a single viewpoint omnidirectional camera with a hyperbolic mirror. In this camera system\cite{scaramuzza2006flexible}, the extrinsic parameters are estimated by solving a two-step least-squares linear minimization problem from the observations of a calibration pattern.

On the other hand, planar mirrors are used for capturing multi-view images. The virtual cameras generated by planar mirrors have identical intrinsic parameters and are time-synchronized to the original camera. These can be a strong advantage for multi-view applications, such as reflectance analysis and 3D shape reconstruction. Mukaigawa \etal\cite{mukaigawa2011hemispherical} introduce a hemispherical confocal imaging using a turtleback reflector, and Inoshita \etal\cite{inoshita13full} use it to measure full-dimensional (8-D) BSSRDF (bidirectional scattering-surface reflectance distribution function). Nane and Nayer\cite{nane98stereo} propose a computational stereo system using a single camera and various types of mirrors, such as planar, ellipsoidal, hyperboloidal, and paraboloidal ones. In the case of a planar mirror, they calibrate this stereo system by computing a fundamental matrix from point correspondences. Gluckman and Nayer\cite{gluckman2001catadioptric} discuss the case of two planar mirrors and propose an efficient calibration method based on the ideas that the relative orientations of virtual cameras are restricted to the class of planar motion.

In the context of kaleidoscopic imaging exploiting high-order multiple reflections, Ihrke \etal\cite{ihrke2012kaleidoscopic} and Reshetouski and Ihrke\cite{reshetouski2013mirrors,reshetouski13discovering,reshetouski11three} have proposed a theory on modeling the chamber detection, segmentation, bounce tracing, shape-from-silhouette, \etc. As introduced in the previous section, the essential problems in this context are {\it chamber assignment} and {\it mirror parameters estimation}.

Reshetouski and Ihrke\cite{reshetouski13discovering} solve the chamber assignment problem by utilizing constraints on an apparent 2D distances between 2D projections for the specific mirror configurations, that is, all mirrors are perpendicular to the ground plane. For automatically assigning chambers in general mirror configurations, our paper introduces the kaleidoscopic projection constraint that provides 3D information of the mirror system and proposed the efficient chamber assignment algorithm using the geometric constraints inspired by Reshetouski and Ihrke \cite{reshetouski13discovering}.

Regarding the mirror parameters estimation, a per-mirror basis approach is utilized in some conventional approaches\cite{ihrke2012kaleidoscopic,reshetouski2013mirrors} without considering their kaleidoscopic, \ie, multiple reflections, relationships. That is they detect a chessboard first\cite{zhang2000flexible} in each of the chambers and then estimate the mirror normals and the distances from the 3D chessboard positions in the camera frame. This per-mirror calibration can be also done by applying the algorithms in Section \ref{sec:related_work_1}. Ying \etal \cite{ying10geometric} realized a calibration algorithm using a geometric relationship of multiple reflections that lie in a 3D circle. Ying \etal's approach can estimate the mirror normals and distances per mirror-pair basis without any reference object whose 3D geometry is known. However, it does not handle reflections by multiple mirrors or fully utilize kaleidoscopic relationships on multiple reflections. 

Compared with these approaches, the proposed method has two main advantages:
\begin{itemize}
\item our chamber assignment can be applied to general mirror poses, and
\item our mirror parameters calibration can estimate parameters satisfying kaleidoscopic projection constraints with explicitly considering multiple reflections by multiple mirrors.
\end{itemize}
Furthermore, our method does not require a reference object of known geometry. That is, our method requires a single 3D point whose 3D position is not given beforehand.

These advantages are realized by introducing a {\it kaleidoscopic projection constraint}. The following sections first define the geometric model of the kaleidoscopic imaging system and then introduce the kaleidoscopic projection constraint.

\section{Problem definition}\label{sec:section_3}
\subsection{Notation and goals}
As illustrated by Figure \ref{fig:mirror_model}, consider a 3D point $\VEC{p}$ and its reflection $\VEC{p}'$ by a mirror $\pi$. Their projections $\VEC{q}$ and $\VEC{q}'$ are given by the perspective
projection:
\begin{equation}
 \lambda \VEC{q} = A \VEC{p}, \quad \lambda' \VEC{q}' = A \VEC{p}', \label{eq:projection}
\end{equation}
where $A$ is the intrinsic matrix of the camera, and $\lambda$ and $\lambda'$ are the depths from the camera. Note that we assume that the intrinsic matrix is known and we use the camera coordinate system as the world coordinate system throughout the paper.

Let $\VEC{n}$ and $d > 0$ denote the normal and the distance of the mirror $\pi$ satisfying
\begin{equation}
 \VEC{n}^\top \VEC{x} + d = 0, \label{eq:planar_constraint}
\end{equation}
where $\VEC{x}$ is a 3D position in the scene. Here the normal vector is oriented to the camera center.

As illustrated in Figure \ref{fig:mirror_model}, the distance $t$ from $\VEC{p}$ and $\VEC{p}'$ to the mirror $\pi$ satisfies:
\begin{equation}
 \VEC{p} = \VEC{p}' + 2 t \VEC{n}.
\end{equation}
The projection of $\VEC{p}'$ to $\VEC{n}$ gives
\begin{equation}
 t + d = - \VEC{n}^\top \VEC{p}'.
\end{equation}
By eliminating $t$ from these two equations, we have
\begin{align}
 & \VEC{p} =-2(\VEC{n}^\top \VEC{p}'+d)\VEC{n}+\VEC{p}', \\
 \Leftrightarrow &  \tilde{\VEC{p}} = S \tilde{\VEC{p}}'
 = \begin{bmatrix}
    H  & -2d\VEC{n}\\
    \VEC{0}_{1\times 3} & 1
   \end{bmatrix}
 \tilde{\VEC{p}}', \label{eq:householder}
\end{align}
where $H$ is a $3 \times 3$ Householder matrix given by $H = I_{3{\times}3} - 2 \VEC{n} \VEC{n}^\top$.  $\tilde{\VEC{x}}$ denotes the homogeneous coordinate of $\VEC{x}$, and $\VEC{0}_{m{\times}n}$ denotes the $m{\times}n$ zero matrix. $I_{n\times n}$ denotes the $n \times n$ identity matrix. Note that this $S$ also satisfies inverse transformation, that is $\tilde{\VEC{p}}' = S \tilde{\VEC{p}}$.

Suppose the camera observes the target 3D point directly and indirectly via $N_{\pi}$ mirrors as shown in Figure \ref{fig:multi_mirror_system}. Let $\VEC{p}_0$ denote the original 3D point and $\VEC{p}_{i}$ denote the first reflection of $\VEC{p}_0$ by the mirror $\pi_i \: (i=1,\cdots, N_{\pi})$ (Figure \ref{fig:kaleidoscopic_image}). The reflection $\VEC{p}_{i}$ satisfies
\begin{align}
\tilde{\VEC{p}}_{0} = S_i \tilde{\VEC{p}}_{i}
  = \begin{bmatrix}
    H_i  & -2d_i\VEC{n}_i\\
    \VEC{0}_{1\times 3} & 1
  \end{bmatrix}
  \tilde{\VEC{p}}_{i}, 
\end{align}
where $\VEC{n}_i$ and $d_i$ denote the mirror normal and its distance respectively and $H_i$ is given by $H_i = I_{3{\times}3} - 2\VEC{n}_i\VEC{n}_i^{\top}$.

\begin{figure}[t]
\centering
\includegraphics[width=0.9\linewidth]{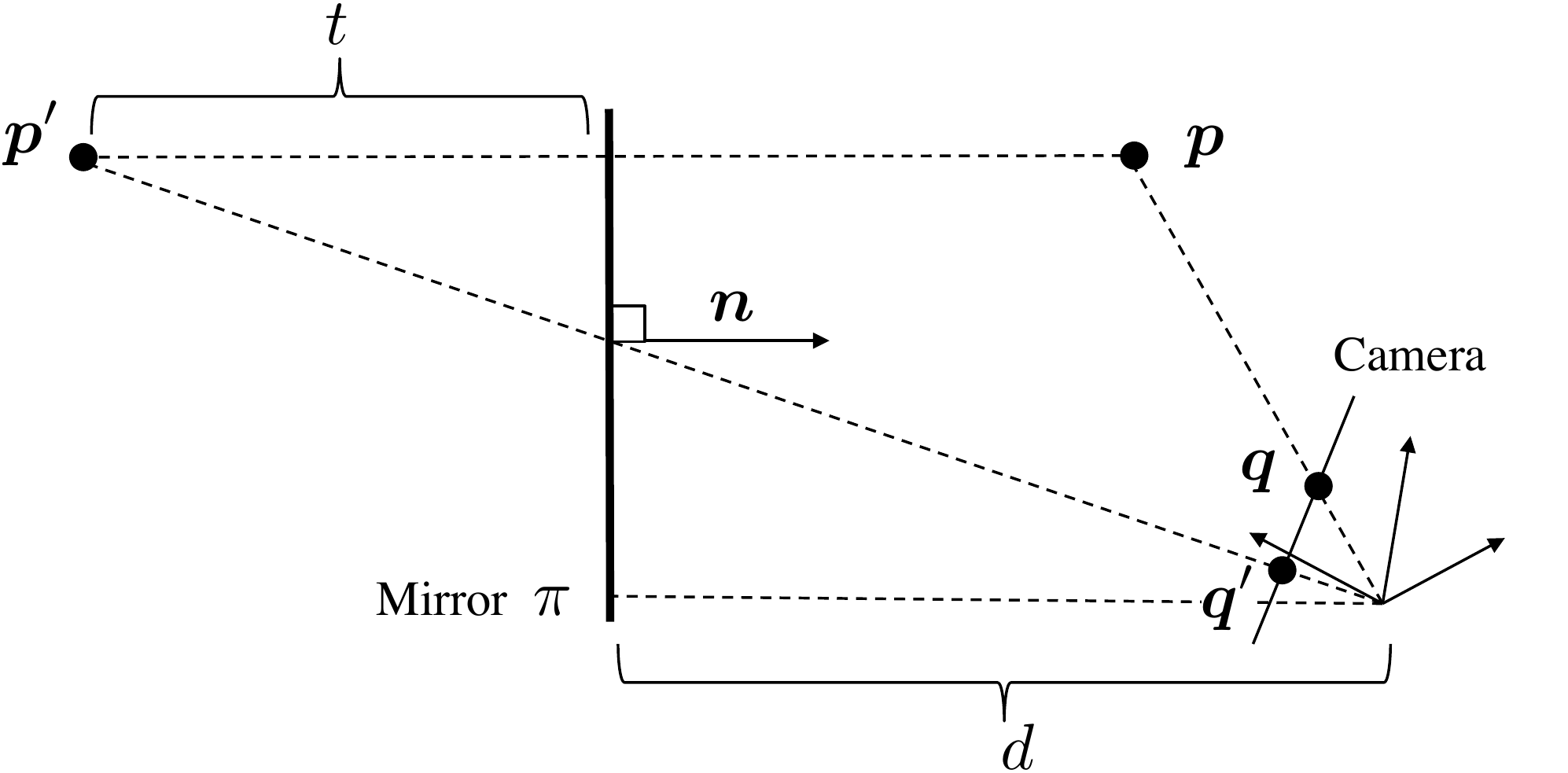}
\caption{The mirror geometry. A mirror $\pi$ of normal $\VEC{n}$ and distance $d$ reflects a 3D point $\VEC{p}$ to $\VEC{p}'$, and they are projected to $\VEC{q}$ and $\VEC{q}'$, respectively.}
\label{fig:mirror_model}
\end{figure}
\begin{figure}[t]
\centering \includegraphics[width=0.9\linewidth]{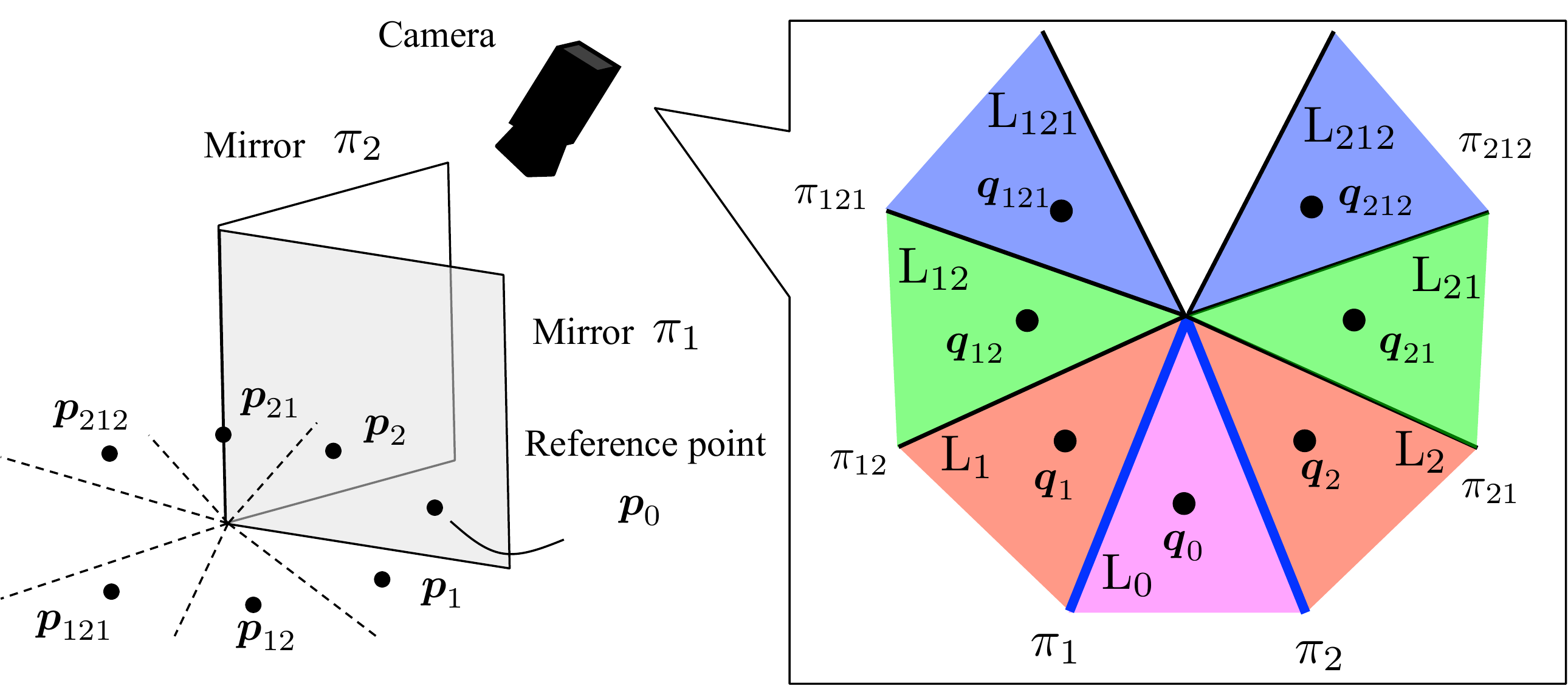}
\caption{Chamber assignment. The magenta region indicates the base chamber. The red, green, and blue regions indicate the chambers corresponding to the first, second, and third reflections, respectively.}
\label{fig:kaleidoscopic_image}
\end{figure}

Furthermore, such mirrors define virtual mirrors as a result of multiple reflections. Let $\pi_{ij} \: (i,j = 1,\cdots, N_{\pi}, \: i \neq j)$ denote the virtual mirror defined as a mirror of $\pi_j$ by $\pi_i$, $\VEC{p}_{ij}$ denote the reflection of $\VEC{p}_i$ by $\pi_{ij}$. As detailed in Hesch \etal\cite{Hesch2010extrinsic}, the multiple reflections matrices $S_{ij}$ and $H_{ij}$ for $\pi_{ij}$ are given by
\begin{equation}
\begin{split}
 S_{ij} & = S_i S_j, \\
 H_{ij} & = H_i H_j. \label{eq:hijhihj}
\end{split}
\end{equation}
The third and further reflections, virtual mirrors, and chambers are defined in the same manner:
\begin{equation}
  \Pi_{k=1}^{N_k} S_{i_k} \: (i_k = 1,2,3, \: i_{k} \neq i_{k+1}),
\end{equation}
where $N_k$ is the number of reflections.

Obviously, the 3D subspaces where $\VEC{p}_0$ and $\VEC{p}_{i}$ can exist are mutually exclusive, and the captured image can be subdivided into regions called \textit{chambers} corresponding to such subspaces. Suppose the perspective projections of $\VEC{p}_x$ are denoted by $\VEC{q}_x \in Q \: (x = 0, 1,\dots, N_{\pi}, 12, 13, \dots)$ in general. In this paper, we denote the 2D region where $\VEC{q}_0$ exists as the \textit{base chamber} $L_0$\cite{reshetouski11three,reshetouski2013mirrors}, and we use $L_x$ to denote the chamber where $\VEC{q}_x$ exists.

Consider a 2D point set $R = \{\VEC{r}_i\}$ detected from the captured image as candidates of $\VEC{q}_x$. To calibrate the kaleidoscopic imaging system extrinsically, two problems need to be solved:
\begin{itemize}
\item to assign the chamber label $L_x$ to $\VEC{r}_i \in R$ to identify to which the chamber each of the projections $\VEC{q}_x$ belongs and,
\item to estimate the parameters of the real mirrors $\pi_i (i=1,\cdots,N_{\pi})$, \ie, normals $\VEC{n}_i$ and distances $d_i$ of them, from kaleidoscopic projections $\VEC{q}_x$.
\end{itemize}

For solving these problems, we utilize a kaleidoscopic projection constraint introduced in \cite{gluckman2001catadioptric, ying10geometric, mariottini2010catadioptric}.

\subsection{Kaleidoscopic projection constraint}\label{sec:section_4}
Suppose the camera observes a 3D point of unknown geometry $\VEC{p}$. The mirror $\pi$ of matrix $S$ defined by the normal $\VEC{n}$ and the distance $d$ reflects $\VEC{p}$ to $\VEC{p}' = S\VEC{p}$ (Eq. \eqref{eq:householder}).

Based on the epipolar geometry\cite{hartley00multiple,ying13self}, $\VEC{n}$, $\VEC{p}$, and $\VEC{p}'$ are coplanar and satisfy
\begin{equation}
 \left(\VEC{n} \times \VEC{p}\right)^\top \VEC{p}' = 0.
\end{equation}
By substituting $\VEC{p}$ and $\VEC{p}'$ by $\lambda A^{-1} \VEC{q}$ and $\lambda' A^{-1} \VEC{q}'$, respectively (Eq. \eqref{eq:projection}), we obtain
\begin{equation}
 \VEC{q}^\top A^{-\top} [\VEC{n}]_{\times}^{\top} A^{-1} \VEC{q}' = 0, \label{eq:coplanarity}
\end{equation}
where $[\VEC{n}]_\times$ denotes the $3 \times 3$ skew-symmetric matrix representing the cross product by $\VEC{n}$, and this is the essential matrix of this mirror-based binocular geometry\cite{ying13self,mariottini2010catadioptric}.

By representing the normalized image coordinates of $\VEC{q}$ and $\VEC{q}'$ by $(x, y, 1)^\top = A^{-1}\VEC{q}$ and $(x', y', 1)^\top = A^{-1}\VEC{q}'$ respectively, Eq. \eqref{eq:coplanarity} can be rewritten as
\begin{equation}
 \begin{pmatrix}
 y - y' & x' - x & x y' - x' y
 \end{pmatrix}
 \VEC{n}
 = 0. \label{eq:kaleidoscopic_projection_constraint}
\end{equation}
We refer to this epipolar constraint on mirror reflections Eq. \eqref{eq:kaleidoscopic_projection_constraint} as the {\it kaleidoscopic projection constraint} in this paper. 

{\bf Single reflection:} Let $\VEC{p}_0$ denote a 3D point and $\VEC{p}_i$ denote the reflection by mirror $\pi_i$. Since $\VEC{p}_i$ is expressed as $\VEC{p}_i = S_i \VEC{p}_0$, the normalized image coordinates of their projections, $\VEC{q}_0$ and $\VEC{q}_i$, obviously satisfy Eq. 
\eqref{eq:kaleidoscopic_projection_constraint} as:
\begin{equation}
 \begin{pmatrix}
   y_0 - y_i & x_i - x_0 & x_0 y_i - x_i y_0
 \end{pmatrix}
 \VEC{n}_i
 = 0, \label{eq:kaleidoscopic_projection_constraint_single}
\end{equation}
where $(x_0, y_0, 1)^\top = A^{-1}\VEC{q}_0$ and $(x_i, y_i, 1)^\top = A^{-1}\VEC{q}_i$.
 
{\bf High-order reflections:} Let $\VEC{p}_{ij} \: (i,j = 1,\cdots, N_{\pi}, \: i \neq j)$ denote the reflection of $\VEC{p}_i$ by $\pi_{ij}$. This $\VEC{p}_{ij}$ can be expressed as $\VEC{p}_{ij} = S_{ij}\VEC{p}_i = S_{i}S_{j}\VEC{p}_0$ based on Eq. \eqref{eq:hijhihj}. Here $\VEC{p}_j = S_j\VEC{p}_0$ holds as well, and we obtain $\VEC{p}_{ij} = S_{i}S_{j}\VEC{p}_0 \Leftrightarrow \VEC{p}_{ij} = S_i\VEC{p}_j$. This equation means that $\VEC{p}_{ij}$ can be recognized as the first reflection of $\VEC{p}_{j}$ by $\pi_i$, and hence the normalized image coordinates of their projections, $\VEC{q}_{ij}$ and $\VEC{q}_j$, also satisfy Eq. \eqref{eq:kaleidoscopic_projection_constraint} as:
\begin{equation}
 \begin{pmatrix}
   y_{j} - y_{ij}' & x_{ij} - x_{j} & x_{j}y_{ij} - x_{ij}y_{j}
 \end{pmatrix}
 \VEC{n}_i
 = 0, \label{eq:kaleidoscopic_projection_constraint_high_order}
\end{equation}
where $(x_{ij}, y_{ij}, 1)^\top = A^{-1}\VEC{q}_{ij}$.

In $N_k$th reflections, we obtain the kaleidoscopic projection constraint between
$\VEC{p}_{i_k} = S_{i_{N_k}}\Pi_{k=1}^{N_k-1} S_{i_k}\VEC{p}_0$ and $\VEC{p}_{i_k'} =
\Pi_{k=1}^{N_k-1} S_{i_k}\VEC{p}_0$ in the same manner.

\section{Chamber assignment}\label{sec:section_5}
Based on the kaleidoscopic projection constraint in the last section, we introduce a new algorithm that identifies the chamber label of each projection. Our algorithm utilizes an analysis-by-synthesis approach that iteratively draws many projections and evaluates their geometric consistency in terms of the kaleidoscopic projection to find the best chamber assignment.

In what follows, the concept of {\it base structure}, \ie, minimal configuration for estimating the real mirror parameters using the kaleidoscopic projection constraint, is introduced. Our algorithm hypothesizes many base structure candidates from observed points and evaluates each of their consistencies as a kaleidoscopic projection. In this evaluation, additional geometric constraints are introduced that reject mirror parameters satisfying the kaleidoscopic projection constraint but are physically infeasible.

Note that we first introduce our algorithm using a two-mirror system (Figure \ref{fig:base_structure}(a)) as an example and then extend it to the general case.

\subsection{Base structure}\label{sec:base_structure}
Suppose $2N_{\pi}$ points of the observed points $R$ are selected and could be hypothesized as $\VEC{q}_0, \VEC{q}_1, \dots$ correctly. The mirror normal $\VEC{n}_i$ has two degrees of freedom and can be linearly estimated by collecting two or more linear constraints on it. In the case of $N_{\pi}=2$, the mirror normal $\VEC{n}_1$ can be estimated as the null space of the coefficient matrix of the following system defined by the kaleidoscopic projection constraint (Eq. \eqref{eq:kaleidoscopic_projection_constraint}) using $\{\PAIR{\VEC{q}_0}{\VEC{q}_1}, \PAIR{\VEC{q}_2}{\VEC{q}_{12}}\}$ in Figure \ref{fig:base_structure} (a):
\begin{equation}
 \begin{bmatrix} y_0 - y_1 & x_0 - x_1 & x_0 y_1 - x_1 y_0 \\ y_2 - y_{12} &
 x_{12} - x_2 & x_0 y_{12} - x_{12}y_0 \\ \end{bmatrix} \VEC{n}_1
 = \VEC{0}_{3{\times}1}, \label{eq:kpc_use}
\end{equation}
where $\PAIR{\VEC{q}}{\VEC{q}'}$ denotes a \textit{doublet}\cite{reshetouski13discovering}, \ie, the pair of projections $\VEC{q}$ and $\VEC{q}'$ for Eq. \eqref{eq:kaleidoscopic_projection_constraint}.

\begin{figure}[t]
\centering \includegraphics[width=0.9\linewidth]{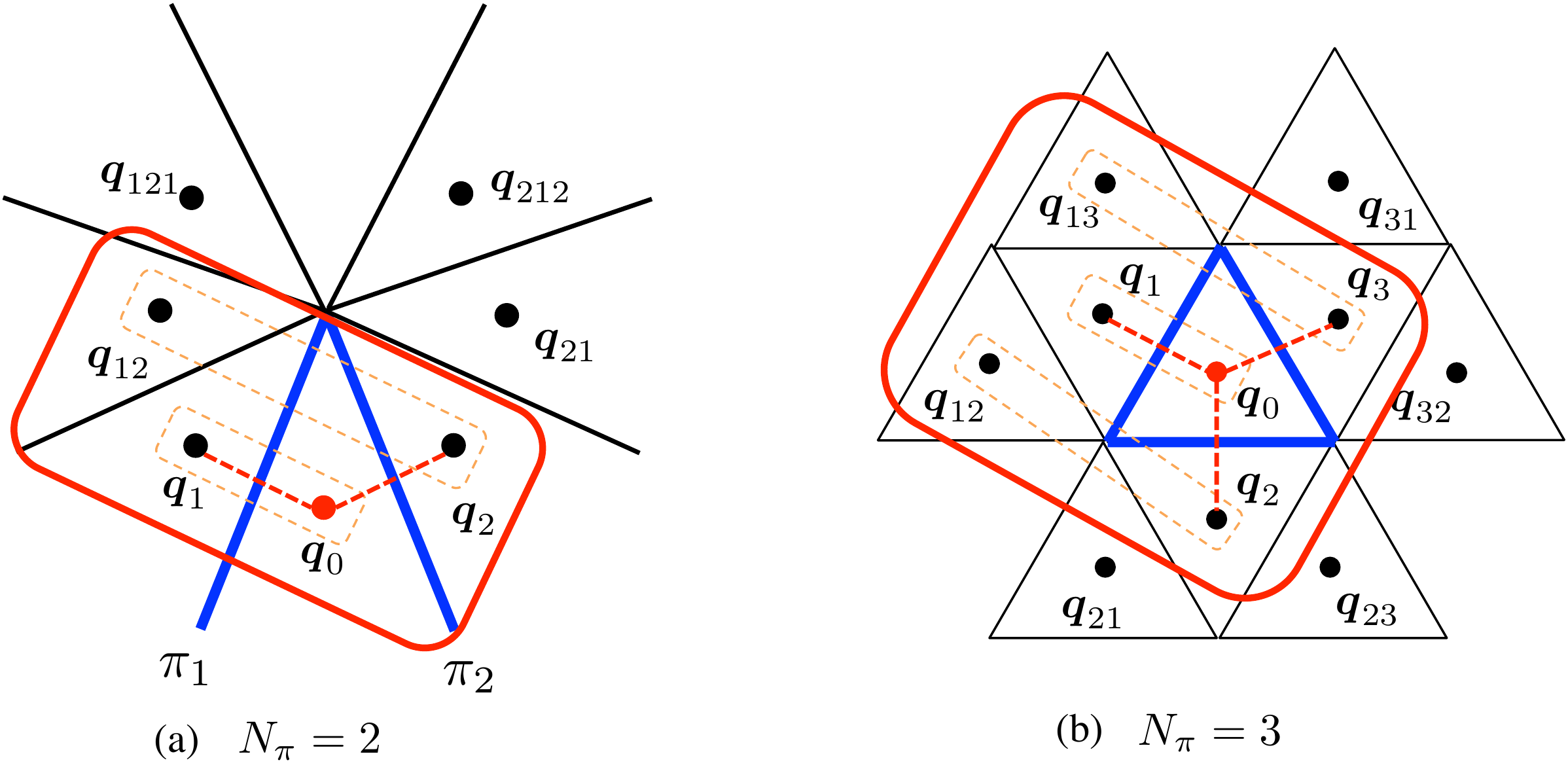}
\caption{The red boxes show examples of base structures in the case of (a) $N_{\pi} = 2$ and (b) $N_{\pi} = 3$. The red point indicates the point assumed as the base chamber and the dotted boxes indicate doublets. The red dotted lines indicate the reflection pairs, and the blue lines indicate the discovered mirrors.}
\label{fig:base_structure}
\end{figure}

\begin{figure}[t]
\centering \includegraphics[width=0.8\linewidth]{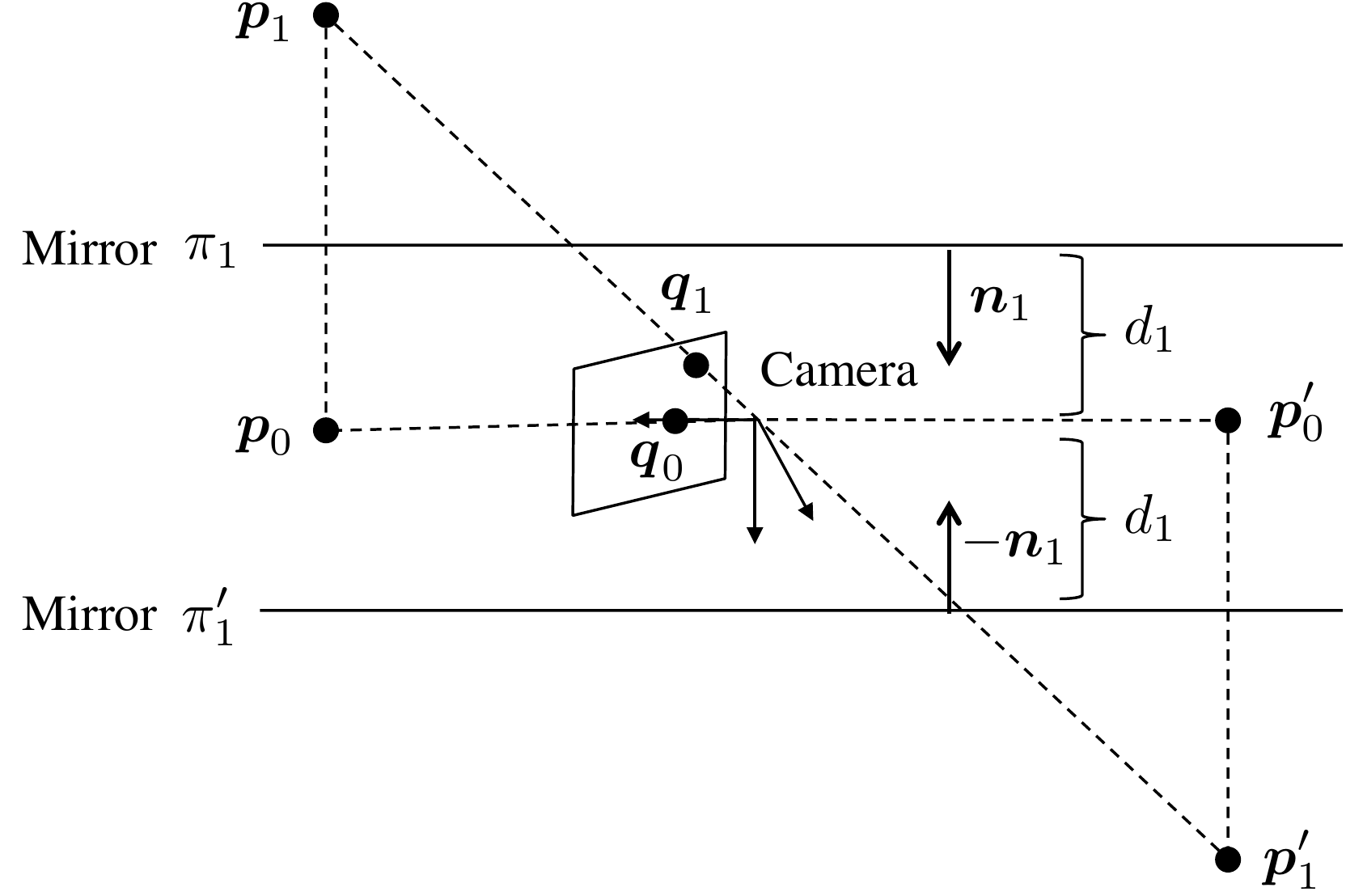}
\caption{Sign ambiguity of $\VEC{n}$ and corresponding triangulations. $\VEC{p}_i (i=0,1)$ and $\VEC{p}'_i$ are estimated from possible mirror parameters of $\pi_1$ and $\pi_1'$, \ie, $(\VEC{n}_1, d_1)$ and $(-\VEC{n}_1, d_1)$, respectively. Both of these mirror parameters satisfy Eq. \eqref{eq:planar_constraint}, and $\VEC{p}'_i$ appears as $- \VEC{p}_i$.}
\label{fig:sign_ambiguity}
\end{figure}

By using the estimated $\VEC{n}_1$ and assuming $d_1=1$ without loss of generality, the 3D point $\VEC{p}_1$ can be described as $\tilde{\VEC{p}}_1 = S_1\tilde{\VEC{p}}_0$ by Eq. \eqref{eq:householder}. By substituting $\VEC{p}_0$ and $\VEC{p}_1$ in this equation by using $\VEC{q}_0$ and $\VEC{q}_1$ as expressed in Eq. \eqref{eq:projection}, the 3D point $\VEC{p}_0$ and $\VEC{p}_1$ can be triangulated by solving the following linear system for $\lambda_0$ and $\lambda_1$:
\begin{align}
 & \tilde{\VEC{p}}_0 = S_1\tilde{\VEC{p}}_1, \\
 \Leftrightarrow & \lambda_0 A^{-1} \VEC{q}_0 = H_1 \lambda_1 A^{-1} \VEC{q}_1 -2\VEC{n}_1, \\
 \Leftrightarrow &
 \begin{bmatrix}
 H_1 A^{-1}\VEC{q}_0 & -A^{-1}\VEC{q}_1 \\
 \end{bmatrix}
 \begin{bmatrix}
 \lambda_0\\
 \lambda_1
 \end{bmatrix}
 = 2\VEC{n}_1. \label{eq:triangulation}
\end{align}

Note that $\VEC{n}_1$ can be determined up to sign from Eq. \eqref{eq:kpc_use}. As illustrated in Figure \ref{fig:sign_ambiguity}, both mirror parameters, $(\VEC{n}_1, d_1)$ and $(-\VEC{n}_1, d_1)$, are possible configurations in terms of Eq. \eqref{eq:planar_constraint}, but one of them triangulates $\VEC{p}_i (i=0,1)$ behind the camera. As a result, by rejecting this configuration, we obtain a unique mirror parameter.

Similarly the 3D points $\VEC{p}_2$ and $\VEC{p}_{12}$ can be triangulated by solving the linear system for $\lambda_2$ and $\lambda_{12}$. Because $\VEC{p}_2$ is the reflection of $\VEC{p}_0$ by the mirror $\pi_2$, the mirror normal $\VEC{n}_2$ as well as the distance $d_2$ can be estimated as
\begin{equation}
 \VEC{n}_2 = \frac{\VEC{p}_0 - \VEC{p}_2}{|\VEC{p}_0 - \VEC{p}_2|}, \quad\quad
 d_2 = - \VEC{n}_2^{\top}\frac{\VEC{p}_0 + \VEC{p}_2}{2}.
 \label{eq:mirror_parameters_from_3D_points}
\end{equation}

This doublets pair $\{\PAIR{\VEC{q}_0}{\VEC{q}_1}, \PAIR{\VEC{q}_2}{\VEC{q}_{12}}\}$ is a minimal configuration for linear estimation of the real mirror parameters in the $N_{\pi}=2$ case, and we call this minimal configuration a {\it base structure} of our chamber assignment. Note that the above doublet pair is not the unique base structure. That is, $\{\PAIR{\VEC{q}_0}{\VEC{q}_2}, \PAIR{\VEC{q}_1}{\VEC{q}_{21}}\}$ is also a base structure for $N_{\pi}=2$ case.

This procedure can be generalized intuitively as an $N_{\pi} (N_{\pi} > 2)$ pair method for the $N_{\pi}$ mirror system as follows. In the case of $N_{\pi}$ mirrors, we can observe a base structure that consists of $N_{\pi}$-tuple of doublets between a 3D point and its reflection by a mirror $\pi_j$ up to the second reflections as $\PAIR{\VEC{q}_0}{\VEC{q}_j}$ and $\PAIR{\VEC{q}_i}{\VEC{q}_{ji}} \: (1 \le j \le N_{\pi} \: i=1,\dots,j-1,j+1,\dots,N_{\pi})$. By using this base structure, $\VEC{n}_j$ can be estimated first, and then $\VEC{p}_i \: (i=0,\dots,N_{\pi})$ as described for the two-mirror case. As a result, all the mirror normals and the distances can be estimated by assuming $d_1=1$.

For example, a base structure of the $N_{\pi} = 3$ system illustrated in Figure \ref{fig:base_structure}(b) is 3-tuple of doublets \{$\PAIR{\VEC{q}_0}{\VEC{q}_1}$, $\PAIR{\VEC{q}_2}{\VEC{q}_{12}}$, $\PAIR{\VEC{q}_3}{\VEC{q}_{13}}$\}. This can be seen as a 6-point algorithm by considering up to the second reflections.

\subsection{Geometric constraints on kaleidoscopic projections}
\begin{figure}[t]
\centering \includegraphics[width=0.7\linewidth]{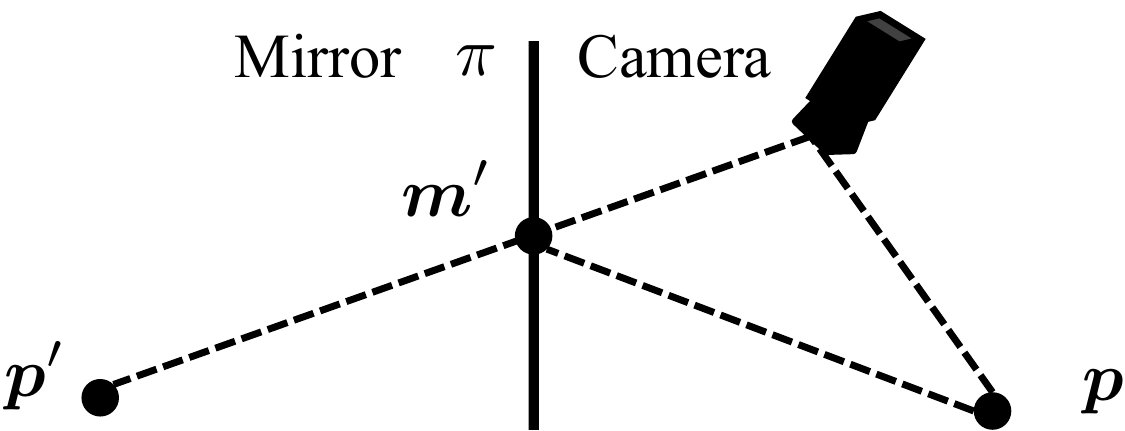}
\caption{The 3D point $\VEC{p}$ is closer to the camera than its reflection $\VEC{p}'$ from triangle inequality.}
\label{fig:lemma_distance}
\end{figure}

\begin{figure}[t]
\centering
\includegraphics[width=0.9\linewidth]{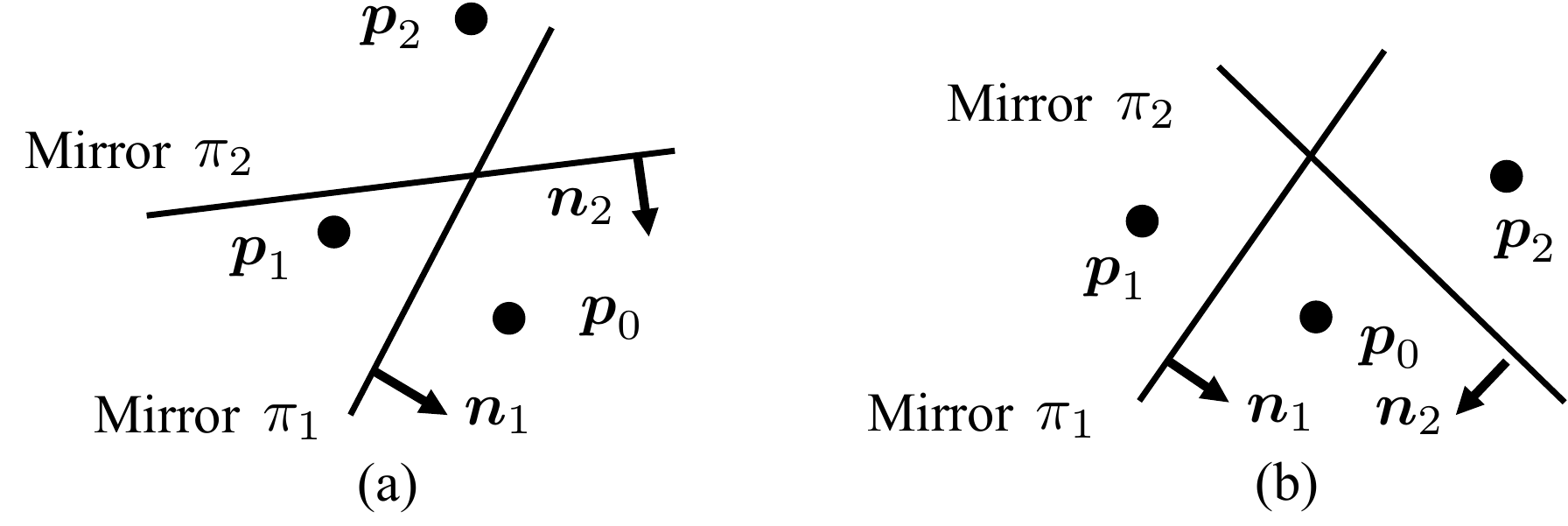}
\caption{To obtain second reflections, the mirrors should be facing each other, $\VEC{n}_i^{\top}\VEC{n}_j < 0$ as in (b). In $\VEC{n}_i^{\top}\VEC{n}_j \ge 0$ cases as in (a), the first reflection $\VEC{p}_2$ on $\pi_2$ is reflected behind mirror $\pi_1$ and cannot be reflected by $\pi_1$ as the second reflection.}
\label{fig:lemma_inverse_reflection}
\end{figure}

The procedure in Section \ref{sec:base_structure} assumes the selected points are correctly labeled as $\VEC{q}_0, \VEC{q}_1, \dots$. This section provides geometric constraints that evaluate the correctness of the hypothesized labeling.

In the mirror parameter estimation in Section \ref{sec:base_structure}, the smallest singular value computed on solving Eq. \eqref{eq:kpc_use} for $\VEC{n}_1$ indicates the feasibility of interpreting the set of $N_{\pi}$ doublets as a base structure:
\begin{equation}
 E = \frac{|e_{N_m}|}{\sum_{i=1,\cdot,N_m} |e_i|},\label{eq:pruning}
\end{equation}
where $e_i$ is the $i$th largest singular value of the coefficient matrix in Eq. \eqref{eq:kpc_use}. That is, once we select $2N_{\pi}$ projections in the image and hypothesize $N_{\pi}$ doublets correctly as a base structure, this should be zero in ideal conditions without noise. If the hypothesis of doublets is not appropriate, this should not be zero. This can be utilized as a geometric condition for removing inadequate hypotheses.

However, $E = 0$ is not a sufficient constraint to conclude the hypothesized base structure is physically feasible in terms of a kaleidoscopic projection. In addition to this condition, the hypothesized base structure should satisfy the following two propositions inspired by Reshetouski and Ihrke \cite{reshetouski13discovering}.
\begin{prop}
The 3D point $\VEC{p}_0$ projected to $\VEC{q}_0$ in the base chamber is the closest point to the camera among its reflections.
\end{prop}
\begin{proof}
As illustrated in Figure \ref{fig:lemma_distance}, the distance to the reflection $\VEC{p}'$ of $\VEC{p}$ is identical to the distance to the point of reflection $\VEC{m}'$ and the distance from $\VEC{m}'$ to $\VEC{p}$. From triangle inequality, $\VEC{p}$ is closer to the camera than $\VEC{p}'$, \ie, $|\VEC{p}| < |\VEC{p}_i|$.

By considering this single reflection for the case of $\VEC{p}_0$ and $\VEC{p}_i$ with the original camera, the case of $\VEC{p}_i$ and $\VEC{p}_{ji}$ with the virtual camera reflected by the mirror $\pi_j$, and so forth, we have $|\VEC{p}_0| < |\VEC{p}_i| < |\VEC{p}_{ji}| < \cdots$.
\end{proof}

\begin{prop}
The mirror normals should satisfy $\VEC{n}_i^{\top}\VEC{n}_j < 0 \: (i \neq j)$.
\end{prop}
\begin{proof}
The mirror parameters estimated in Section \ref{sec:base_structure} require a projection of a second reflection. As illustrated in Figure \ref{fig:lemma_inverse_reflection}, this is identical to guarantee that the mirrors are facing each other: \begin{equation} \VEC{n}_i^{\top}\VEC{n}_j < 0. \end{equation}
\end{proof}

These propositions are based on the same ideas described in Reshetouski and Ihrke \cite{reshetouski13discovering} and can be recognized as the extended version of them. Proposition 1 rejects configurations as illustrated in Figure \ref{fig:error_example_Nm_2_2D}(a) where $\VEC{q}_1$ is wrongly interpreted as the base chamber. Note that proposition 1 is only true for points that are not on the mirror. Proposition 2 rejects configurations as illustrated in Figure \ref{fig:error_example_Nm_2_2D}(b) where $\VEC{q}_0$ is correctly interpreted as the base chamber but $\pi_2$ reflects $\VEC{p}_{212}$ behind $\pi_1$. 

One of the key contributions of this paper in the chamber assignment problem is to implement such geometric constraints in a 3D space by utilizing the kaleidoscopic projection constraint, while Reshetouski \etal \cite{reshetouski13discovering} implement them in a 2D space. As a result, their 2D formulation requires the mirrors to be orthogonal to a ground plane, while our 3D formulation does not.

\begin{figure}[t]
\centering
\includegraphics[width=0.9\linewidth]{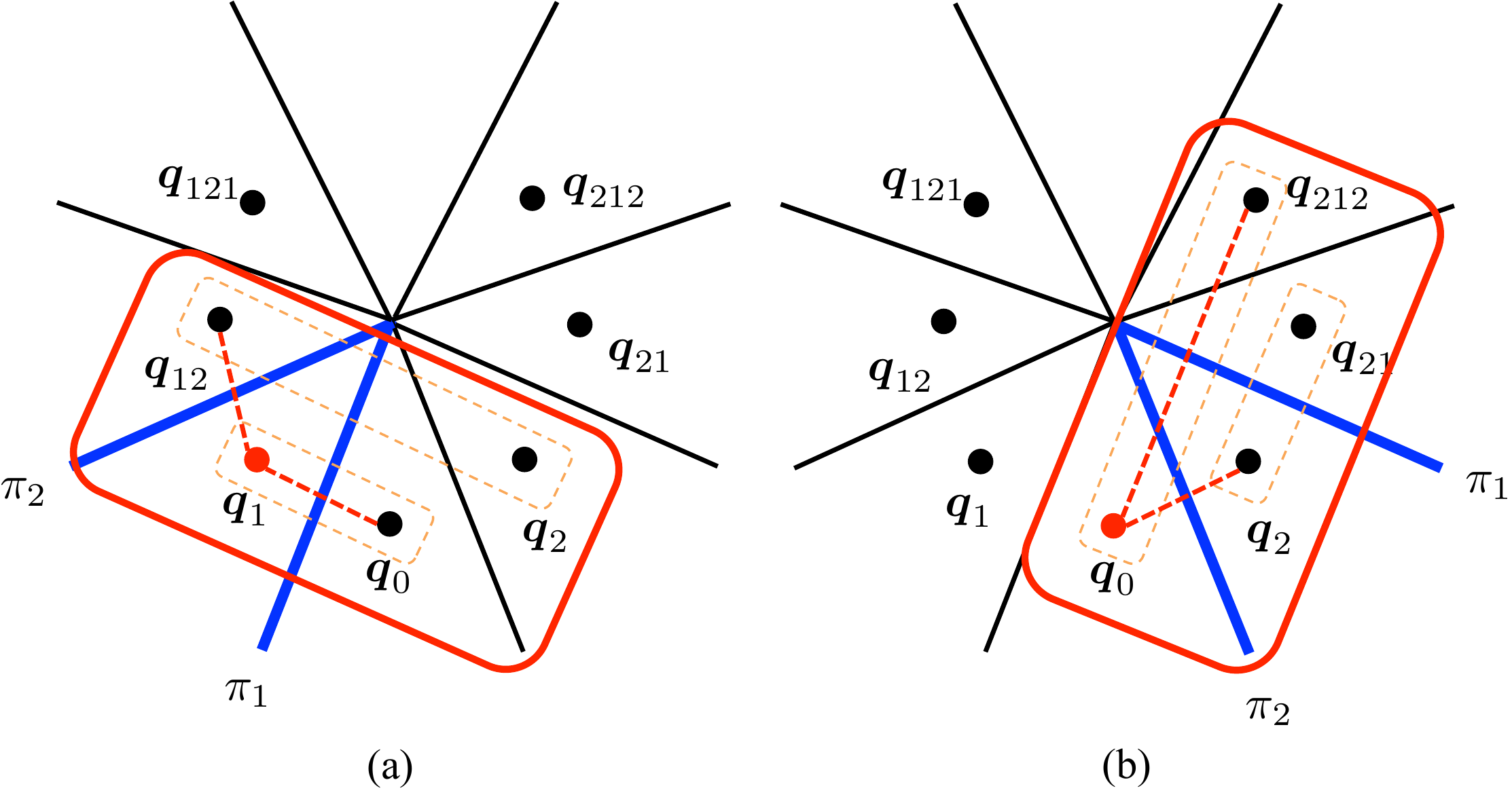}
\caption{Configurations satisfying Eq. \eqref{eq:kpc_use} but not (a) Proposition 1 or (b) Proposition 2. The color codes are identical to Figure \ref{fig:kaleidoscopic_image}.}
\label{fig:error_example_Nm_2_2D}
\end{figure}

\begin{algorithm*}[t]
  \caption{Chamber assignment algorithm}
  \label{alg:overview}
  \begin{algorithmic}
    \REQUIRE $\VEC{r}_i (i = 0, 1, \cdots, |R|)$
    \ENSURE $\mathrm{\Lambda}_\mathrm{out} = \{ L_{\VEC{r}_0}, L_{\VEC{r}_1}, \dots, L_{\VEC{r}_{|R|}} \}$
    \FORALL{base structure candidates}
    \STATE{compute mirror parameters $\VEC{n}_j, d_j (j = 1, \cdots, N_{\pi})$ by solving Eq \eqref{eq:kpc_use} and Eq \eqref{eq:mirror_parameters_from_3D_points}.}
    \IF{Eq. \eqref{eq:pruning} $\neq$ 0}
    \STATE{continue;}
    \ENDIF
    \IF{${}^\exists i, |\VEC{p}_0| > |\VEC{p}_{i}|$}
    \STATE{continue; \# Proposition 1}
    \ENDIF
    \IF{$\VEC{n}^{\top}_i\VEC{n}_j > 0$}
    \STATE{continue; \# Proposition 2}
    \ENDIF
    \STATE{assign chamber labels for each $\VEC{r}_i$ as described in Section \ref{sec:assignment} and obtain label set $\mathrm{\Lambda}_r$.}
    \STATE{compute the recall score $\mathcal{R}$ of $\mathrm{\Lambda}_r$ by Eq. \eqref{eq:recall}.}
    \ENDFOR
    \STATE{select the label set with a highest recall score as $\mathrm{\Lambda}_\mathrm{out}$.}
  \end{algorithmic}
\end{algorithm*}

\subsection{Discontinuity-aware label propagation}\label{sec:assignment_algorithm}
Suppose a hypothesized base structure satisfies the above two propositions. This section introduces an algorithm that propagates the labeling to the projections not involved in the hypothesized base structure by synthesizing projections of all possible reflected points.

Given the mirror parameters $\VEC{n}_i, d_i \: (i = 1,\dots,N_{\pi})$ and the triangulated 3D point $\VEC{p}_0$, the $k$th reflection and its projection can be computed by Eqs. \eqref{eq:householder} and \eqref{eq:projection}. However, if such reflection is projected outside of the corresponding chamber, it is not observable from the camera and cannot generate further reflections. This is called \textit{discontinuity}\cite{reshetouski11three}.

Such discontinuity can simply be intuitively detected by examining the chamber label at the projected pixel, but this requires knowing the pixel-wise chamber labeling that is not available up to this point. Instead, we introduce another detection approach based on 3D geometry.

\subsubsection{Discontinuity detection}
Consider a multiply-reflected 3D point $\tilde{\VEC{p}}_{i_{k} \dots i_1} = \Pi_{j=i_1, \dots, i_k} S_j \tilde{\VEC{p}}_0$. As pointed out in Reshetouski \etal \cite{reshetouski11three}, if this is visible from the camera, the ray from the camera center to $\VEC{p}_{i_{k} \dots i_1}$ should intersect with the mirror of the first reflection $\pi_{i_1}$ as illustrated in Figure \ref{fig:discontinuity}.

Let $\ell$ denote the ray to the target point. Since the system has $N_{\pi}$ mirror planes, the above condition can be evaluated by computing the $N_{\pi}$ intersections between each of the planes and $\ell$, and by testing, if the intersection with the mirror $\pi_{i_1}$ in question is the closest among the intersections in front of the camera.

\subsubsection{Label propagation for a hypothesized base structure}\label{sec:assignment}
Suppose projections of the visible reflections by considering the discontinuity are synthesized as $\hat{Q} = \{ \hat{\VEC{q}}_{i_k \dots i_1} \} \: (k \ge 1, 1 \le i_x \le N_{\pi})$. 

The goal of the label propagation is to find correspondences between the synthesized point set $\hat{Q}$ and the observed point set $R = \{\VEC{r}_i\}$ as a sort of bipartite matching. Suppose the matching cost to associate $\hat{\VEC{q}}_{i_k \dots i_1}$ and $\VEC{r}_i$ is modeled by the 2D distance between them. The matching should minimize the total matching cost
\begin{equation}
 \mathcal{E} = \sum_{\hat{\VEC{q}}_{i_k \dots i_1} \in \hat{Q}} ||\hat{\VEC{q}}_{i_k \dots i_1} - \hat{\VEC{r}}_{i_k \dots i_1}||,\label{eq:label_propagation}
\end{equation}
where $\hat{\VEC{r}}_{i_k \dots i_1} \in R$ is the point selected as the corresponding point of $\hat{\VEC{q}}_{i_k \dots i_1}$ by assigning label $L_{i_k \dots i_1}$. 

Since solving this combination optimization for each of the trials in our analysis-by-synthesis is computationally expensive, we approximated this process by the nearest neighbor search that simply assigns the nearest candidate. Note that we ignored the assignment of $\hat{\VEC{q}}_{i_k \dots i_1}$ whose distance is larger than a threshold in computing Eq. \eqref{eq:label_propagation}. As a result, doublets can share a candidate point in $R$ by multiple synthesized points in $\hat{Q}$, but we found this approximation is acceptable to some extent because of the sparse distribution of the points in $R$ and $\hat{Q}$. This point is discussed later in Section \ref{sec:discussion_chamber_assignment}.

On the basis of this assignment, we can introduce a recall ratio $\mathcal{R}$ that measures how many of the synthesized projections that are supposed to be visible have been assigned detected points:
\begin{equation}
 \mathcal{R} = \frac{|R_c|}{|\hat{Q}|}, \label{eq:recall}
\end{equation}
where $R_c \subseteq R$ is the set of detected points assigned labels, $|\hat{Q}|$ and $|R_c|$ are the sizes of the set $\hat{Q}$ and $R_c$, respectively.

\begin{figure}[t]
\centering
\includegraphics[width=\linewidth]{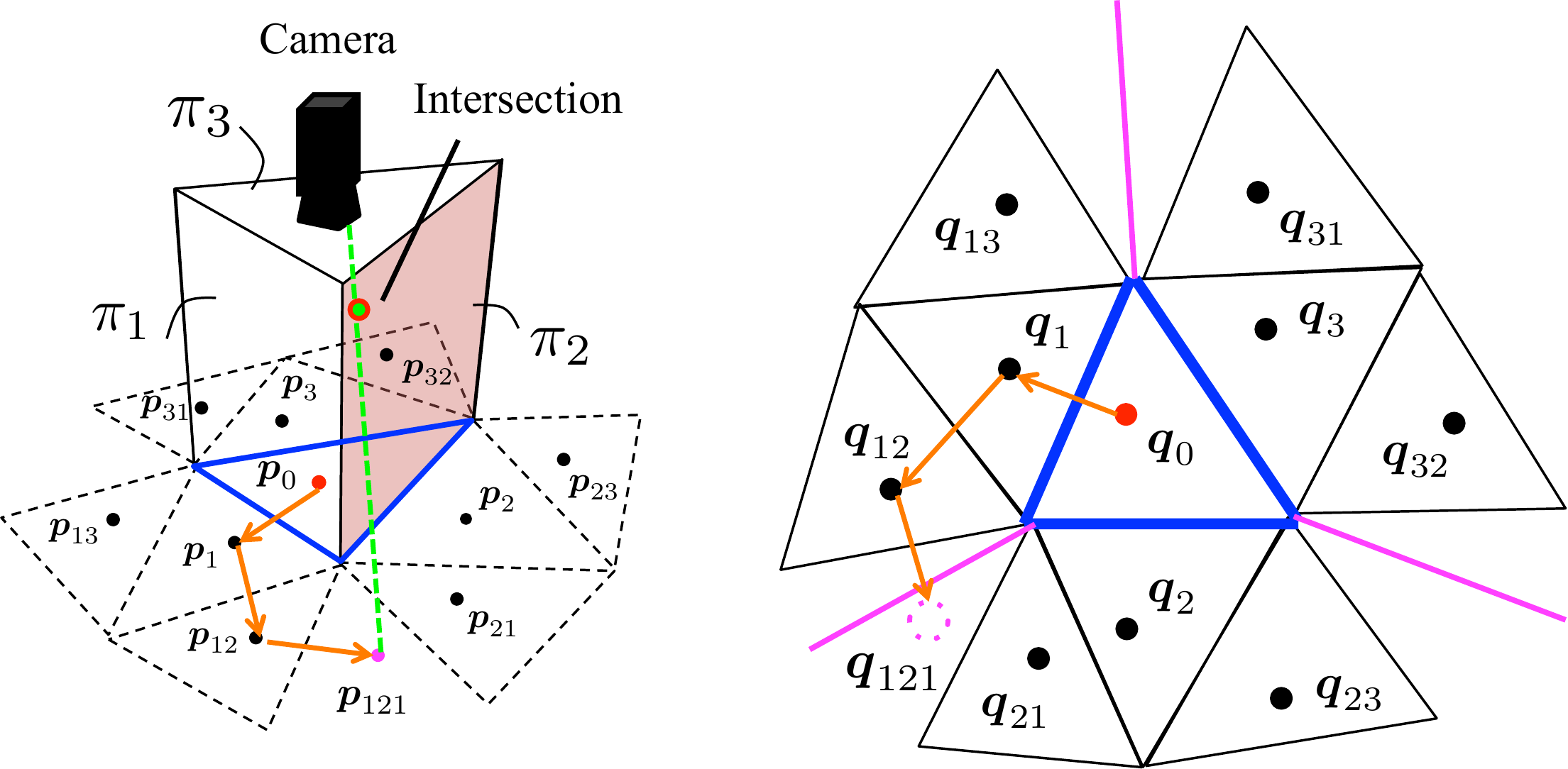}
\caption{Discontinuity. The third reflection $\VEC{p}_{121}$ is not visible from the camera since the viewing ray (the green dotted line in the left image) intersects with not the mirror $\pi_1$ but $\pi_2$. The boundary of such visibility appears as the discontinuity boundary in the image (the magenta lines in the right image)\cite{reshetouski11three}.}
\label{fig:discontinuity}
\end{figure}

\subsection{Chamber assignment algorithm}\label{sec:algrithm}
Algorithm \ref{alg:overview} shows the flow of the proposed chamber assignment algorithm. For each possible base structure, it examines if it satisfies the kaleidoscopic projection constraint expressed by Eq. \eqref{eq:kpc_use}, the base chamber constraint by Proposition 1, and the mirror angle constraint by Proposition 2. Once the base structure passes these verifications, its recall ratio $\mathcal{R}$ is computed by Eq. \eqref{eq:recall}. Finally, the best estimate of the chamber assignment is returned by finding the base structure of the highest recall ratio.

Note that this algorithm evaluates all possible combinations to thoroughly check the behavior of the proposed geometric constraints in Section \ref{sec:base_structure}. This approach is inspired by Reshetouski and Ihrke \cite{reshetouski13discovering}. Possible efficient implementations and effects of the above pruning are discussed in Sections \ref{sec:discussion_ransac} and \ref{sec:discussion_pruning}, respectively.

\section{Estimation of mirror parameters}\label{sec:section_6}
This section introduces a novel algorithm of mirror parameters estimation given the chamber assignment for kaleidoscopic projections of a single 3D point. 

Although the algorithm in Section \ref{sec:base_structure} can also estimate the mirror parameters, it is a per-mirror estimation and is not guaranteed to estimate mirror parameters consistent with projections of higher-order reflections. Instead of such mirror-wise estimations, this section proposes a new linear algorithm that calibrates the kaleidoscopic mirror parameters simultaneously by observing a single 3D point in the scene.

Note that the algorithm is first introduced by utilizing up to the second reflections but can be extended to third or further reflections intuitively as described later.

\subsection{Mirror normals}\label{sec:normal}
As illustrated in Figure \ref{fig:corr_proposed} (a), suppose a 3D point $\VEC{p}_0$ is projected to $\VEC{q}_0$ in the base chamber, and its mirror $\VEC{p}_i$ by $\pi_i$ is projected to $\VEC{q}_i$ in the chamber $L_i$. Likewise, the second mirror $\VEC{p}_{ij}$ by $\pi_{ij}$ is projected to $\VEC{q}_{ij}$ in the chamber $L_{ij}$, and so forth. Here, kaleidoscopic projection constraints are satisfied by two pairs of projections on each mirror $\pi_1$ and $\pi_2$. From these constraints, $\VEC{n}_1$ and $\VEC{n}_2$ can be estimated by solving
\begin{equation}
    \begin{bmatrix}
    y_0 - y_1 & x_1 - x_0 & x_0 y_1 - x_1 y_0 \\
    y_2 - y_{12} & x_{12} - x_2 & x_2 y_{12} - x_{12} y_2 \\
    \end{bmatrix}
    \VEC{n}_1 = \VEC{0}_{3{\times}1}. \label{eq:n1}
\end{equation}
and
\begin{equation}
    \begin{bmatrix}
    y_0 - y_2 & x_2 - x_0 & x_0 y_2 - x_2 y_0 \\
    y_2 - y_{21} & x_{21} - x_1 & x_1 y_{21} - x_{21} y_1 \\
    \end{bmatrix}
    \VEC{n}_2 = \VEC{0}_{3{\times}1}. \label{eq:n2}
\end{equation}

An important observation in this simple algorithm is the fact that (1) this is a linear algorithm even though it utilizes multiple reflections, and (2) the estimated normals $\VEC{n}_1$ and $\VEC{n}_2$ are considered to be consistent with each other even though they are computed on a per-mirror basis apparently.

The first point is realized by using not the multiple reflections of a 3D position but their 2D projections. Intuitively, a reasonable formalization of kaleidoscopic projection is to define a real 3D point in the scene, and then to express that each of the projections of its reflections by Eq. \eqref{eq:householder} coincides with the observed 2D position as introduced in Section \ref{sec:ba} later. This expression, however, is nonlinear in the normals $\VEC{n}_i \: (i=1,2)$ (\eg, $\VEC{p}_{12} = S_1 S_2 \VEC{p}_0$). On the other hand, projections of such multiple reflections can be associated as a result of a single reflection by Eq. \eqref{eq:kaleidoscopic_projection_constraint_high_order} directly (\eg, $\VEC{n}_1$ with $\VEC{q}_{12}$ and $\VEC{q}_2$ as the projections of $\VEC{p}_{12}$ and $S_2\VEC{p}_0$, respectively). As a result, we can utilize 2D projections of multiple reflections in the linear systems above.

This explains the second point as well. The above constraint on $\VEC{q}_{12}$, $\VEC{q}_2$ and $\VEC{n}_1$ in Eq. \eqref{eq:n1} assumes $\VEC{p}_2 = S_2\VEC{p}_0$ being satisfied, and this assumption on the independent mirror normals $\VEC{n}_2$ appears in the second row of \eqref{eq:n1}. Similarly, on estimating $\VEC{n}_2$ by Eq. \eqref{eq:n2}, the assumption $\VEC{p}_1 = S_1\VEC{p}_0$ is implicitly considered in the second row of Eq. \eqref{eq:n2}.

Note that this algorithm can be extended to third or further reflections intuitively. For example, if $\VEC{p}_{21}$ and its reflection by $\pi_1$ is observable as $\lambda_{121}\VEC{q}_{121} = A \VEC{p}_{121} = A S_1 \VEC{p}_{21}$, then it provides
\begin{equation}
 \left( y_{21} - y_{121}, x_{121} - x_{21}, x_{21} y_{121} - x_{121} y_{21} \right) \VEC{n}_1 = 0,
\end{equation}
and can be integrated with Eq. \eqref{eq:n1}.

Also, this algorithm can be extended to $N_{\pi} \ge 3$ cases. In case of $N_{\pi}=3$, for example, we solve
\begin{equation}
    \begin{bmatrix}
    y_0 - y_1 & x_1 - x_0 & x_0 y_1 - x_1 y_0 \\
    y_2 - y_{12} & x_{12} - x_2 & x_2 y_{12} - x_{12} y_2 \\
    y_3 - y_{13} & x_{13} - x_3 & x_3 y_{13} - x_{13} y_3 \\
    \end{bmatrix}
    \VEC{n}_1 = \VEC{0}_{3{\times}1},
\end{equation}
\begin{equation}
    \begin{bmatrix}
    y_0 - y_2 & x_2 - x_0 & x_0 y_2 - x_2 y_0 \\
    y_3 - y_{23} & x_{23} - x_3 & x_3 y_{23} - x_{23} y_3 \\
    y_2 - y_{21} & x_{21} - x_1 & x_1 y_{21} - x_{21} y_1 \\
    \end{bmatrix}
    \VEC{n}_2 = \VEC{0}_{3{\times}1},
\end{equation}
and
\begin{equation}
    \begin{bmatrix}
    y_0 - y_3 & x_3 - x_0 & x_0 y_3 - x_3 y_0 \\
    y_1 - y_{31} & x_{31} - x_1 & x_1 y_{31} - x_{31} y_1 \\
    y_2 - y_{32} & x_{32} - x_2 & x_2 y_{32} - x_{32} y_2
    \end{bmatrix}
    \VEC{n}_3 = \VEC{0}_{3{\times}1}, \label{eq:n3}
\end{equation}
instead of Eqs. \eqref{eq:n1} and \eqref{eq:n2} from point correspondences in Figure \ref{fig:corr_proposed} (b).

\subsection{Mirror distances}
Once the mirror normals $\VEC{n}_1$ and $\VEC{n}_2$ are given linearly, the mirror distances $d_1$ and $d_2$ can also be estimated linearly as follows.

\begin{figure}[t]
\centering
\includegraphics[width=0.9\linewidth]{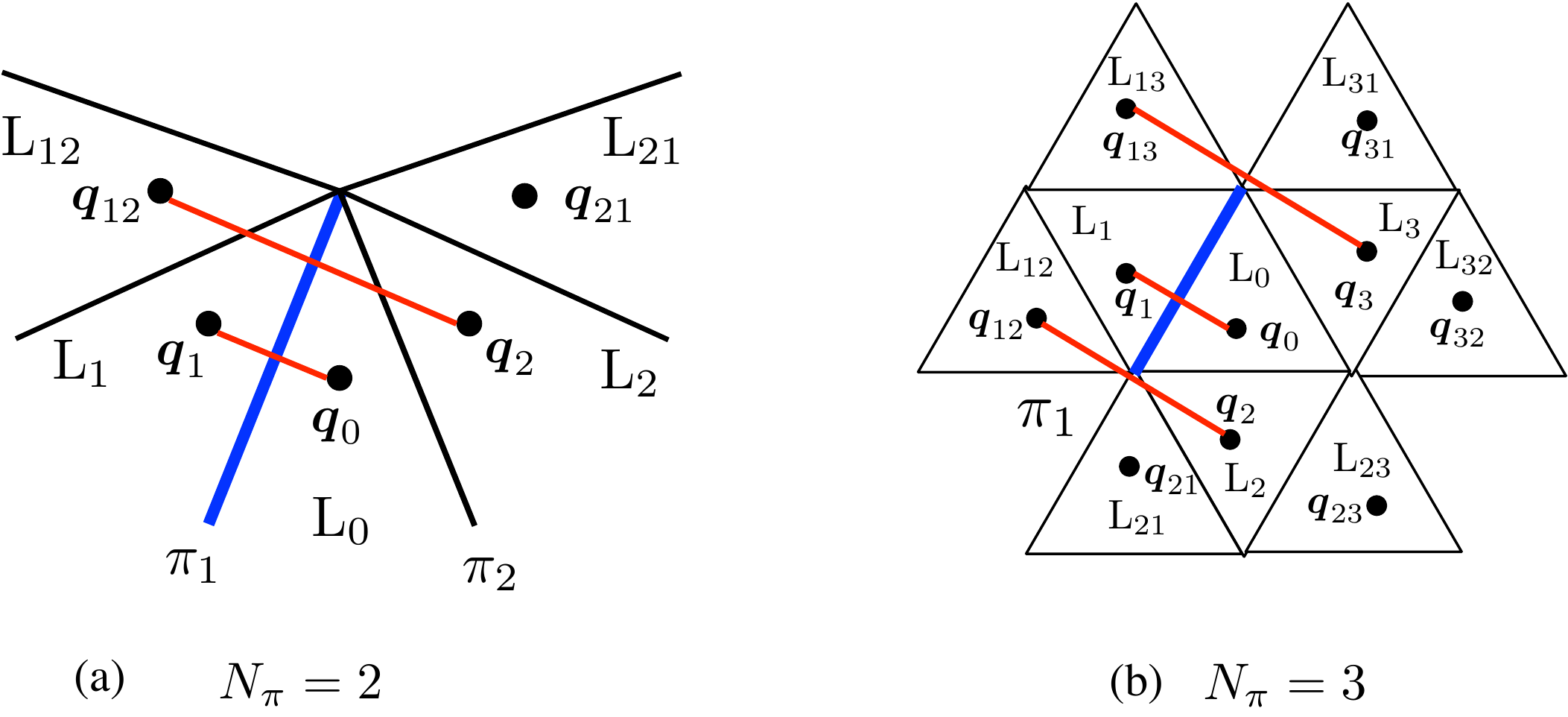}
\caption{Corresponding points in (a) $N_{\pi} = 2$ and (b) $N_{\pi} = 3$ case. (a) Two pairs $\langle\VEC{q}_0,\VEC{q}_1\rangle$ and $\langle\VEC{q}_2,\VEC{q}_{12}\rangle$ (red) are available on mirror $\pi_1$ (blue). (b) Three pairs $\langle\VEC{q}_0,\VEC{q}_1\rangle$,$\langle\VEC{q}_2,\VEC{q}_{12}\rangle$ and $\langle\VEC{q}_3,\VEC{q}_{13}\rangle$(red) are available on mirror $\pi_1$ (blue)}
\label{fig:corr_proposed}
\end{figure}

\subsubsection{Kaleidoscopic re-projection constraint}
The perspective projection Eq. \eqref{eq:projection} indicates that a 3D point $\VEC{p}_i$ and its projection $\VEC{q}_i$ should satisfy the collinearity constraint:
\begin{equation}
 (A^{-1} \VEC{q}_i) \times \VEC{p}_i = \VEC{x}_i \times \VEC{p}_i = \VEC{0}_{3{\times}1}, \label{eq:cross_i}
\end{equation}
where $\VEC{x}_i = \begin{pmatrix} x_i & y_i & 1 \end{pmatrix}^\top$ is the normalized camera coordinate of $\VEC{q}_i$ as defined in Section \ref{sec:section_4}. Since the mirrored points $\VEC{p}_i \: (i=1,2)$ are then given by Eq. \eqref{eq:householder} as
\begin{equation}
    \begin{split}
    \VEC{p}_i & 
    = H_i \VEC{p}_0 - 2 d_i \VEC{n}_i 
    ,
    \label{eq:pi}
    \end{split}
\end{equation}
we obtain
\begin{equation}
\begin{split}
    \VEC{x}_i \times \VEC{p}_i & =  \VEC{x}_i \times (H_i \VEC{p}_0 - 2 d_i \VEC{n}_i), \\
    & = [ \VEC{x}_i ]_\times 
    \begin{bmatrix}
    H_i
    &
    -2 \VEC{n}_i
    \end{bmatrix}
    \begin{bmatrix}
    \VEC{p}_0 \\
    d_i
    \end{bmatrix},\\
    & = \VEC{0}_{3{\times}1}.
\end{split}
\end{equation}

Similarly, the second reflection $\VEC{p}_{ij}$ is also collinear with its projection $\VEC{q}_{ij}$:
\begin{equation}
\begin{split}
    & (A^{-1}\VEC{q}_{ij}) \times \VEC{p}_{ij}, \\
    = & [ \VEC{x}_{ij} ]_\times (H_i \VEC{p}_{j} - 2 d_i \VEC{n}_i), \\
    = & [ \VEC{x}_{ij} ]_\times \left(H_i \left( H_j \VEC{p}_0 - 2 d_j \VEC{n}_j \right) -2 d_i  \VEC{n}_i\right), \\
    = &
    [ \VEC{x}_{ij} ]_\times
    \begin{bmatrix}
    H_i H_j & 
    - 2 \VEC{n}_i &
    - 2 H_i \VEC{n}_j
    \end{bmatrix}
    \begin{bmatrix}
    \VEC{p}_0 \\
    d_i \\
    d_j
    \end{bmatrix}, \\
    = & \VEC{0}_{3{\times}1}. \label{eq:cross_ij}
\end{split}
\end{equation}

By using these constraints, we obtain a linear system of $\VEC{p}_0$, $d_1$ and $d_2$:
\begin{equation}
\begin{split}
    &
    \begin{bmatrix}
        [\VEC{x}_0]_\times & \VEC{0}_{3{\times}1} & \VEC{0}_{3{\times}1} \\
        h_{1} & -2 [ \VEC{x}_{1} ]_\times \VEC{n}_1 & \VEC{0}_{3{\times}1} \\
        h_{2} & \VEC{0}_{3{\times}1} & -2 [ \VEC{x}_{2} ]_\times \VEC{n}_2\\
        h'_{1,2} & -2 [ \VEC{x}_{12} ]_\times \VEC{n}_1 & -2 h''_{1,2}\\
        h'_{2,1} & -2 h''_{2,1} & -2 [ \VEC{x}_{21} ]_\times \VEC{n}_2\\
    \end{bmatrix}
    \begin{bmatrix}
        \VEC{p}_0 \\
        d_1 \\
        d_2
    \end{bmatrix}, \\
    = & 
    K \begin{bmatrix}
        \VEC{p}_0 \\
        d_1 \\
        d_2
    \end{bmatrix} = \VEC{0}_{15{\times}1}, \label{eq:reproj}
\end{split}
\end{equation}
where $h_{i} = [ \VEC{x}_{i} ]_\times H_i$, $h'_{i,j} = [ \VEC{x}_{ij} ]_\times H_i H_j$, $h''_{i,j} = [ \VEC{x}_{ij} ]_\times H_i \VEC{n}_j$.  By computing the eigenvector corresponding to the smallest eigenvalue of $K^\top K$, $(\VEC{p}_0, d_1, d_2)^\top$ can be determined up to a scale factor. In this paper, we choose the scale that normalizes $d_1 = 1$.

Note that Eq. \eqref{eq:reproj} apparently has 15 equations, but only 10 of them are linearly independent. This is simply because each cross product by Eqs. \eqref{eq:cross_i} and \eqref{eq:cross_ij} has only two independent constraints by definition.

Also, as discussed in Section \ref{sec:normal}, the above algorithm can be extended to third or further reflections and $N_{\pi} \ge 3$ cases as well. In $N_{\pi}=3$, considering the reflection of $\VEC{p}_{23}$ by $\pi_1$ as $\lambda_{123}\VEC{q}_{123} = A \VEC{p}_{123} = A S_1 \VEC{p}_{23}$, we have
\begin{equation}
    [\VEC{x}_{123}]_\times
    \begin{bmatrix}
    (H_1 H_2 H_3)^\top \\
    -2\VEC{n}_1^\top \\
    -2 (H_1 \VEC{n}_2)^\top \\
    -2 (H_1 H_2 \VEC{n}_3)^\top
    \end{bmatrix}^\top
    \begin{bmatrix}
        \VEC{p}_0 \\
        d_1 \\
        d_2 \\
        d_3
    \end{bmatrix}
    = \VEC{0}_{3{\times}1}.
\end{equation}

\subsection{Kaleidoscopic bundle adjustment}\label{sec:ba}
The 3D position of $\VEC{p}_0$ can be estimated by solving Eq. \eqref{eq:reproj}. By reprojecting this $\VEC{p}_0$ to each of the chambers as
\begin{equation}
\begin{split}
    \lambda \hat{\VEC{q}}_0 & = A \VEC{p}_0, \\
    \lambda \hat{\VEC{q}}_{i} & = A S_i \VEC{p}_0 \: (i=1,2), \\
    \lambda \hat{\VEC{q}}_{i,j} & = A S_i S_j \VEC{p}_0 \: (i,j=1,2, \: i \neq j),
\end{split}
\end{equation}
we obtain a reprojection error as
\begin{equation}
  \VEC{E}(\VEC{p}_0, \VEC{n}_1, \VEC{n}_2, d_1, d_2)  = 
    \begin{bmatrix}
        \VEC{q}_0 - \hat{\VEC{q}}_0,
        \VEC{e}_1,
        \VEC{e}_2,
        \VEC{e}'_{1,2},
        \VEC{e}'_{2,1},
    \end{bmatrix}^\top,
\label{eq:reproj_error_vec}
\end{equation}
where $\VEC{e}_i = \VEC{q}_i - \hat{\VEC{q}}_i$ and $\VEC{e}'_{i,j} = \VEC{q}'_{i,j} - \hat{\VEC{q}}'_{i,j}$. By minimizing $|| \VEC{E}(\cdot) ||^2$ nonlinearly over $\VEC{p}_0, \VEC{n}_1, \VEC{n}_2, d_1, d_2$, we obtain a best estimate of the mirror normals and the distances.

Compared with the earlier version of this work \cite{takahashi17linear}, we parameterize the $\VEC{p}_0$ and minimize their reprojection errors explicitly by following the bundle adjustment manner presented in \cite{hartley00multiple}.

\section{Experiments}\label{sec:section_7}
This section provides evaluations of the proposed method in terms of chamber assignment and mirror parameter estimation with synthesized and real data.

\subsection{Chamber assignment}\label{sec:eval_chamber_assignment}
\subsubsection{Experimental environment}
The performance of the chamber labeling is evaluated in the following $N_{\pi} = 2$ and $N_{\pi} = 3$ scenarios (Figure \ref{fig:multi_mirror_system}).
\begin{enumerate}
\item two-mirror system using up to third reflections $\VEC{r}_i (i = 0, \cdots, 6)$ (Figure \ref{fig:simulation_setup}(a))
\item three-mirror system using up to second reflections $\VEC{r}_i (i = 0, \cdots, 9)$ (Figure \ref{fig:simulation_setup}(b))
\end{enumerate}
In the latter case with three mirrors, the mirrors are tilted at 5 degrees approximately in order to evaluate the performance with mirrors non-orthogonal to each other. In both scenarios, we used a camera of $1600\times1200$ resolution whose intrinsic matrix $A = [1000, 0, 800; 0, 1000, 600; 0, 0, 1]$.

The performance is evaluated by the accuracy of labeling defined by
\begin{equation}
 E_m = N_{m:\mathrm{correct}} / N_{m},
\end{equation}
where $N_{m}$ is the number of $m$th reflections and $N_{m:\mathrm{correct}}$ is the number of projections labeled correctly.

\begin{figure}[t]
\centering
\includegraphics[width=0.95\linewidth]{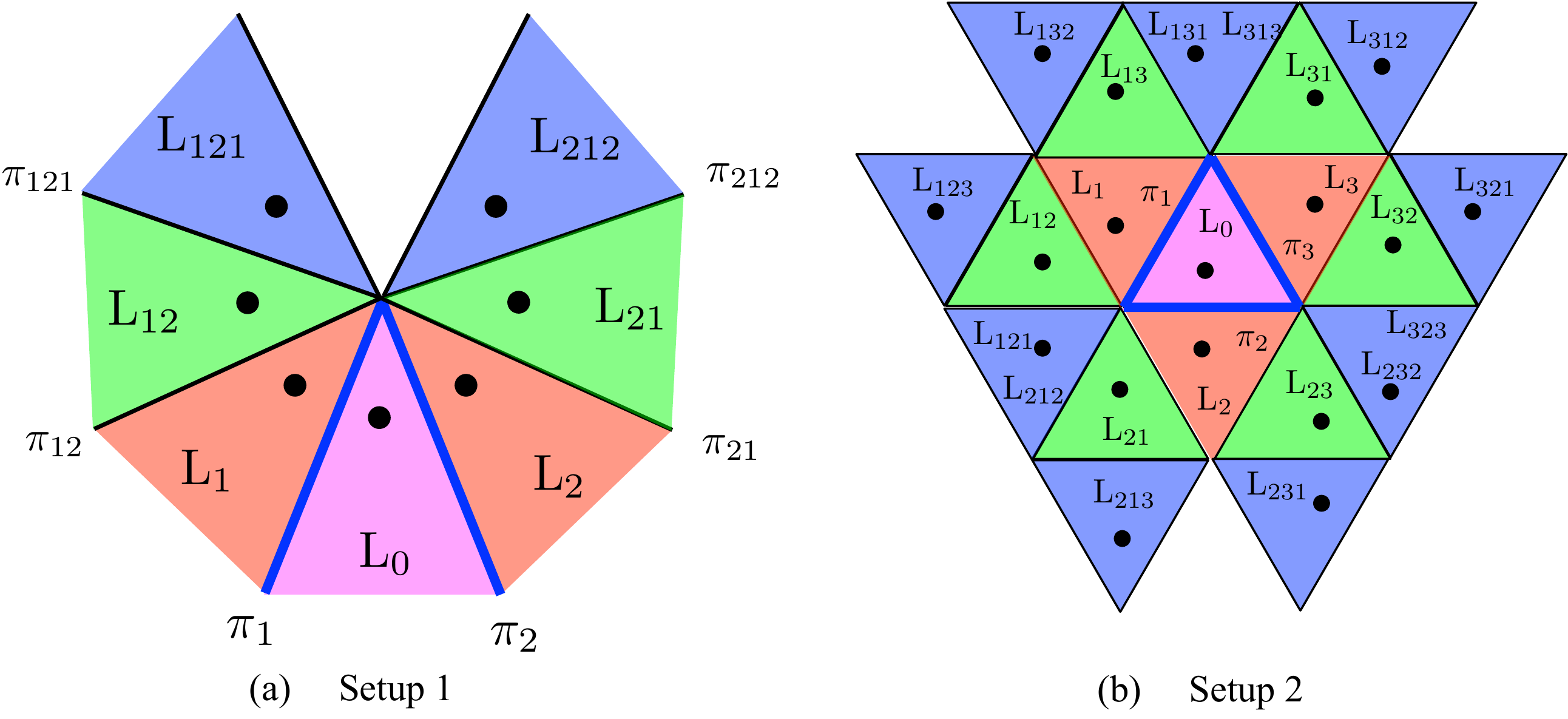}
\caption{Mirror configurations. The red, green and blue regions indicate the chambers corresponding to the first, second, and third reflections, respectively. The bold blue lines indicate the real mirror positions.}
\label{fig:simulation_setup}
\end{figure}

\subsubsection{Quantitative evaluations with synthesized data}
Figures \ref{fig:simulation_Nm_2} and \ref{fig:simulation_Nm_3} report the average accuracy of our labeling in cases of $N_{\pi} = 2$ and $N_{\pi} = 3$, respectively, under different conditions: (a) with Proposition 1 only, (b) with Proposition 2 only, and (c) with Propositions 1 and 2. In these figures, the magenta, red, green, and blue plots indicate the accuracy of labeling 0th, 1st, 2nd, and 3rd reflections, respectively. $\sigma_{\VEC{r}}$ denotes the standard deviation of zero-mean Gaussian pixel noise injected to the positions of the input points $\VEC{r} \in R$, and the average accuracy is computed from the results of 50 trials at each noise level.

\begin{figure*}[t]
\centering
\includegraphics[width=0.9\linewidth]{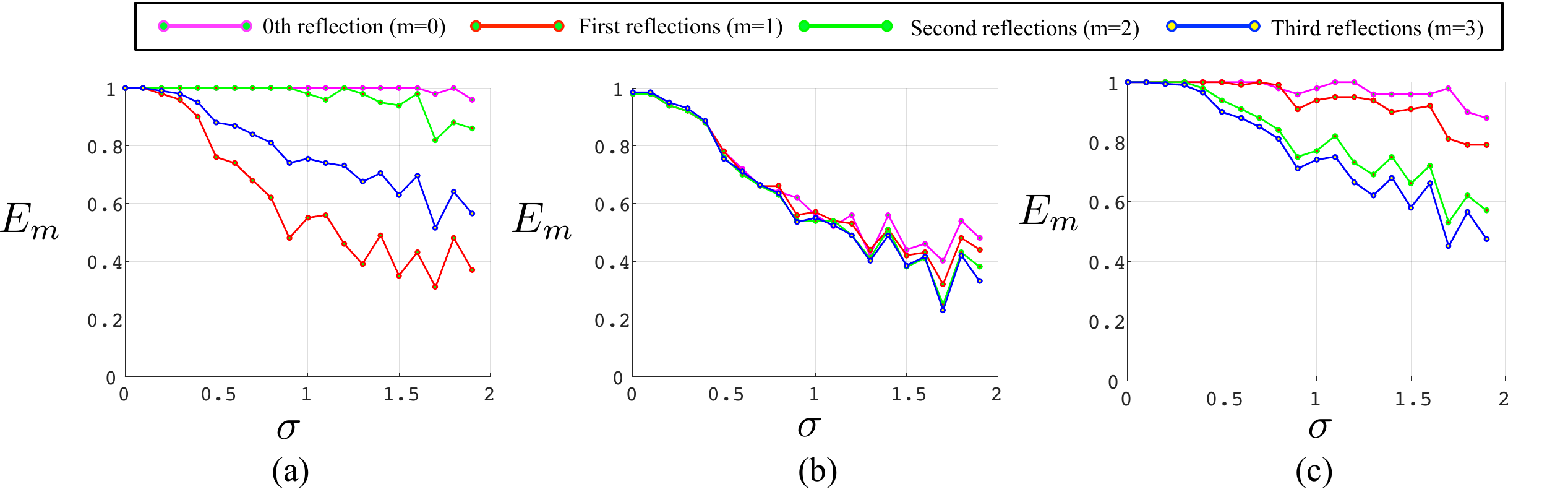}
\caption{The average accuracy of labeling in $N_{\pi}=2$ scenario. (a) with Proposition 1, (b) with Proposition 2, and (c) with Propositions 1 and 2}
\label{fig:simulation_Nm_2}
\end{figure*}

\begin{figure*}[t]
\centering
\includegraphics[width=0.9\linewidth]{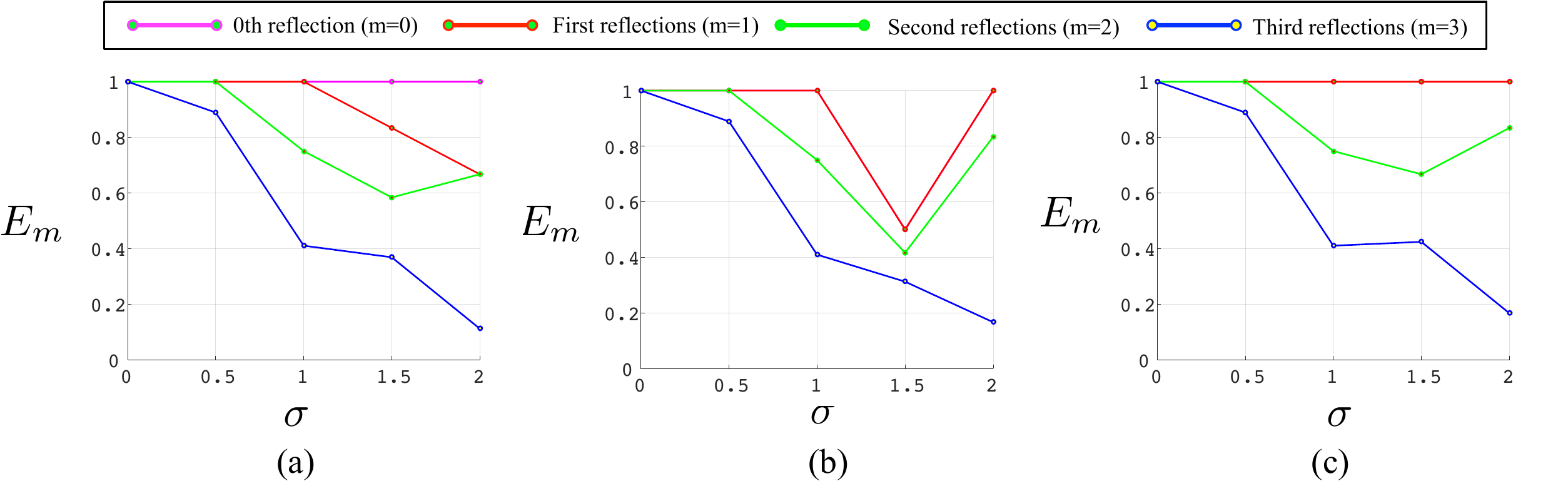}
\caption{The average accuracy of labeling in $N_{\pi}=3$ scenario. (a) with Proposition 1, (b) with Proposition 2, and (c) with Propositions 1 and 2}
\label{fig:simulation_Nm_3}
\end{figure*}

Figure \ref{fig:error_example_Nm_2} shows failure cases in Figures \ref{fig:simulation_Nm_2}(a) and \ref{fig:simulation_Nm_2}(b). In Figure \ref{fig:error_example_Nm_2}(a), the mirror $\pi_1$ is reconstructed between $\VEC{r}_2$ and $\VEC{r}_4$, and $\pi_2$ is reconstructed between $\VEC{r}_0$ and $\VEC{r}_2$. These mirrors correspond to $\pi_{21}$ and $\pi_{2}$ in the original configuration (Figure \ref{fig:simulation_setup}(a)), and such chamber assignment can result in a good reprojection error and a good recall ratio (Eq. \eqref{eq:recall}) but violating Proposition 2. In case of Figure \ref{fig:error_example_Nm_2}(b), the labeling is valid in terms of the reprojection error and the mirror angle but not in terms of Proposition 1.

\begin{figure}[t]
\centering
\includegraphics[width=0.9\linewidth]{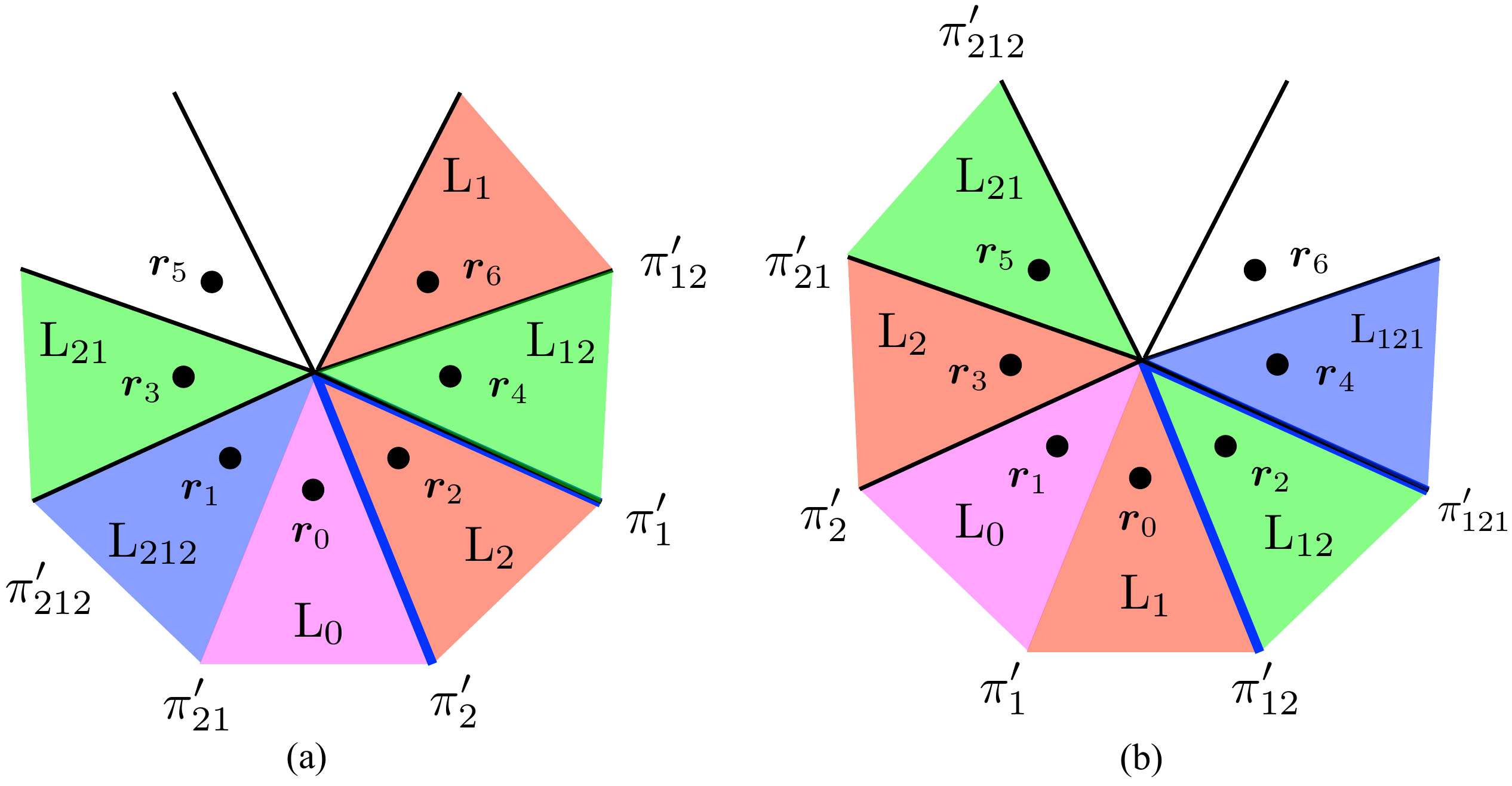}
\caption{Failure cases. (a): The mirror $\pi_1'$ is reconstructed between $\VEC{r}_2$ and $\VEC{r}_4$, and $\pi_2'$ is reconstructed between $\VEC{r}_0$ and $\VEC{r}_2$. (b): The labeling is not valid in terms of the base chamber selection.}
\label{fig:error_example_Nm_2}
\end{figure}

\begin{figure}[htbp]
\centering
\includegraphics[width=0.9\linewidth]{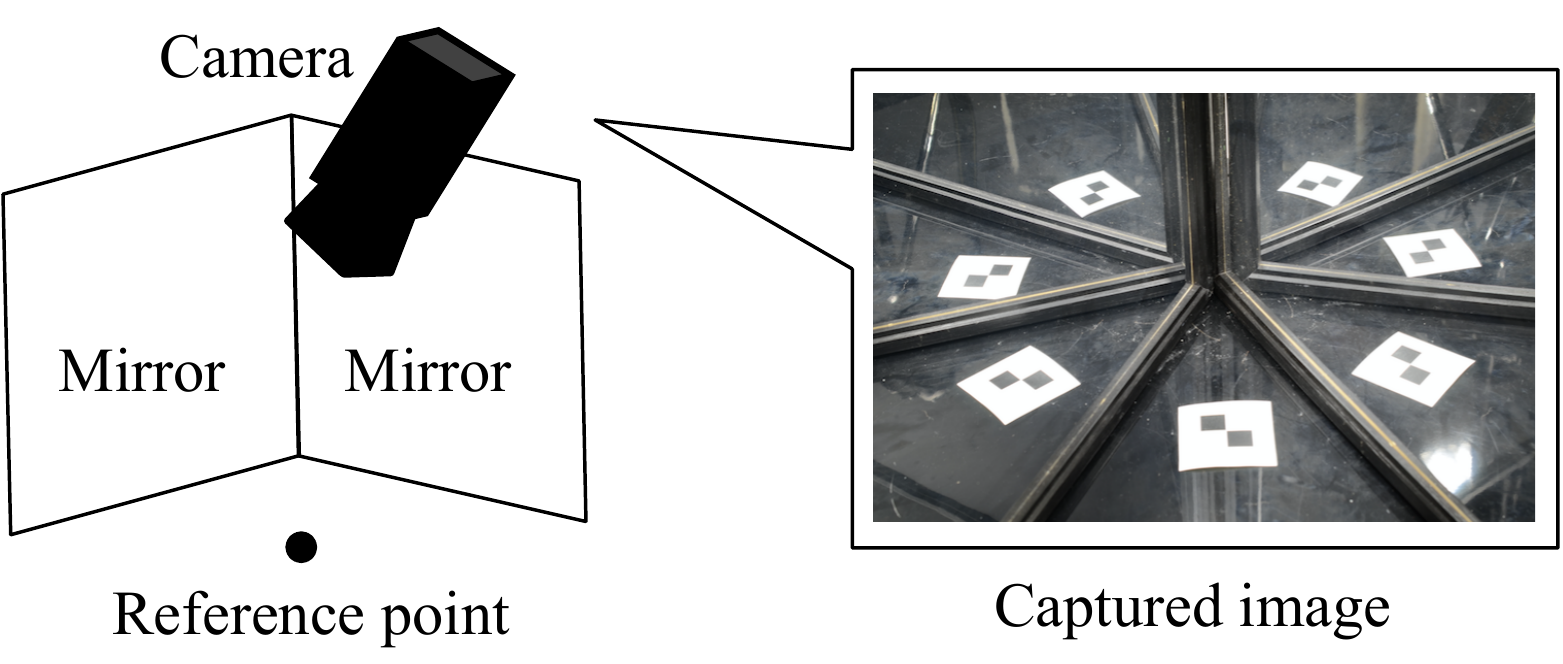}
\caption{A illustration of a capture system (left) and a captured image (right). It consists of planar first surface mirrors and a camera.}
\label{fig:real_data_setup}
\end{figure}

These results show that (1) our method can estimate the correct labeling in the ideal case without noise, and (2) the two propositions can improve the accuracy. 
In addition, these results prove experimentally that our method can work with non-orthogonal mirrors whereas the state-of-the-art algorithm\cite{reshetouski13discovering} assumes mirrors to be orthogonal to a common ground plane.

\begin{figure*}[t]
\centering
\includegraphics[width=0.9\linewidth]{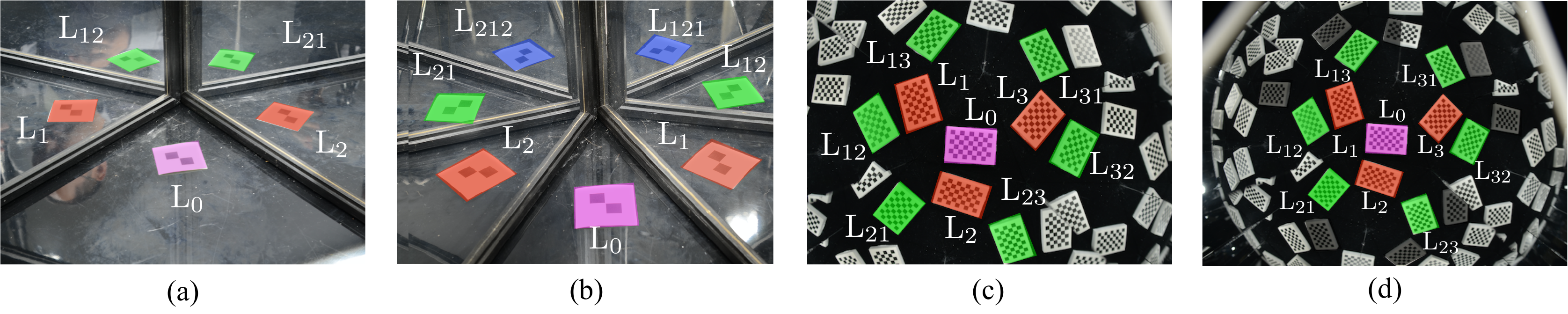}
\caption{Chamber assignment results. The labels denote the assigned chambers, and the magenta, red, green, and blue are superimposed on the direct, the first, the second, and the third reflections.}
\label{fig:real_data_results}
\end{figure*}

One of the main reasons for the degraded accuracy in the noisy condition is the performance of the mirror parameter estimation defined in Section \ref{sec:base_structure}. Since it minimizes the number of input points, the accuracy of the estimated mirror parameters can be sensitive to noise, and hence the projections of the reflections computed using such mirror parameters can fall far from the expected candidate points.

\begin{figure}[t]
\centering
\includegraphics[width=0.6\linewidth]{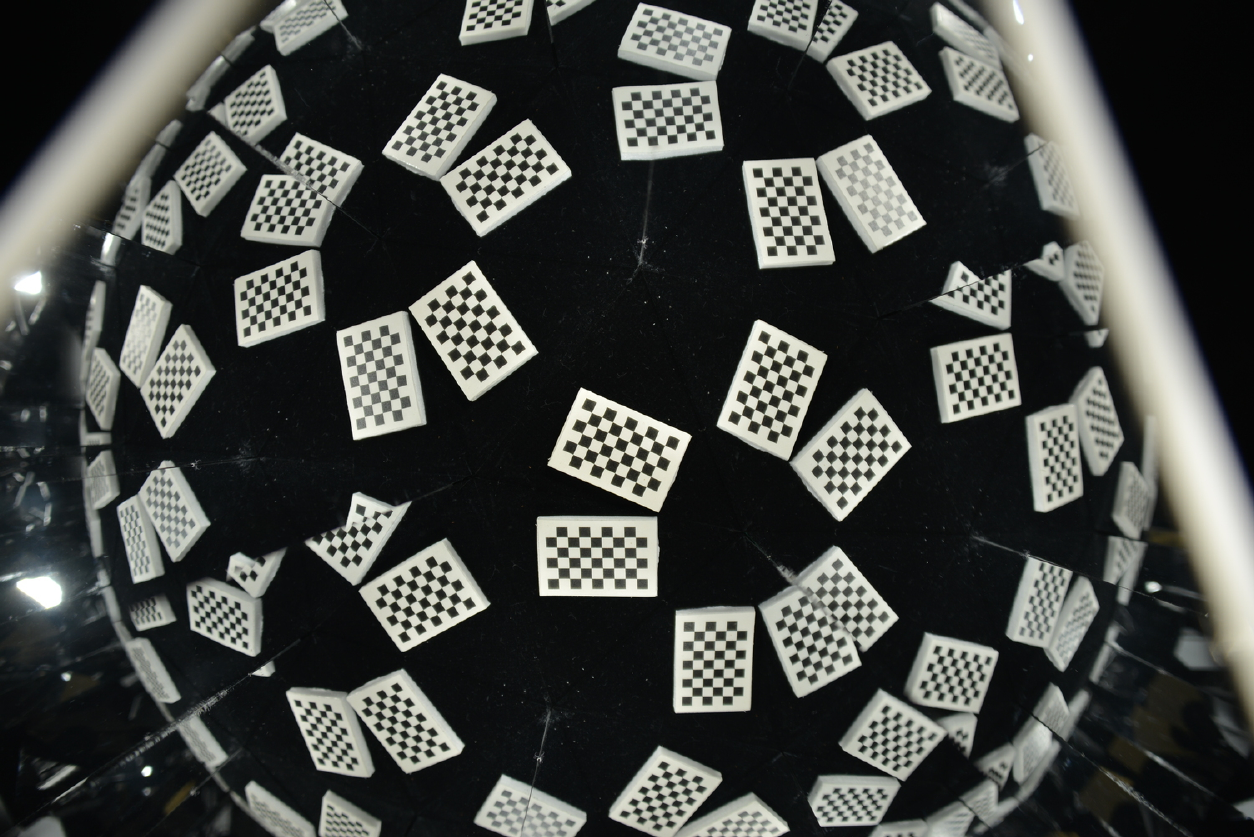}
\caption{A capture of a chessboard used as the reference object for conventional methods.}
\label{fig:chessboard}
\end{figure}

\subsubsection{Qualitative evaluations with real data}
Figure \ref{fig:real_data_setup} shows our mirror-based imaging system. We used planar first surface mirrors and captured images with Nikon D600 (resolution $6016{\times}4016$). The intrinsic parameter $A$ of the camera is calibrated by Zhang's method \cite{zhang2000flexible} in advance. Note that we used a chessboard in order to obtain the precise corner points and gave them to the proposed method as input after removing the label information.

Figure \ref{fig:real_data_results} shows the chamber assignment results by our method for different mirror numbers and different numbers of reflections. We used projections of a single corner point in the images. The labels $L_0, L_1, \dots$ indicate the assigned chambers, and the target objects are superimposed by colors in accordance with the number of reflections. In the case of $N_{\pi} = 3$, each of the mirrors is tilted at about 5 degrees. 

In the two mirrors cases ((a) and (b)), we can see that the proposed method estimates up to the third reflections correctly. While it fails to estimate the third reflections due to the observation noise in the three mirror cases ((c) and (d)), it estimates up to the second reflections correctly, which is needed for the calibration step.

From these results, we can observe that our method can successfully estimate the mirror system structure automatically in practice.

\subsection{Mirror parameter estimation}
This section provides evaluations of the performance of the proposed method in terms of mirror parameter estimation. In this evaluation, we used the same camera used in Section \ref{sec:eval_chamber_assignment}.

The proposed method was compared with the following conventional algorithms. The baseline and orthogonality constraint-based method utilize a reference object of known geometry as shown in Figure \ref{fig:chessboard}. Since the orthogonality constraint-based approach \etal\cite{takahashi12new} require more than two mirrors, we evaluate in $N_{\pi}=3$ configuration.

\begin{figure}[t]
\centering
\includegraphics[width=0.5\linewidth]{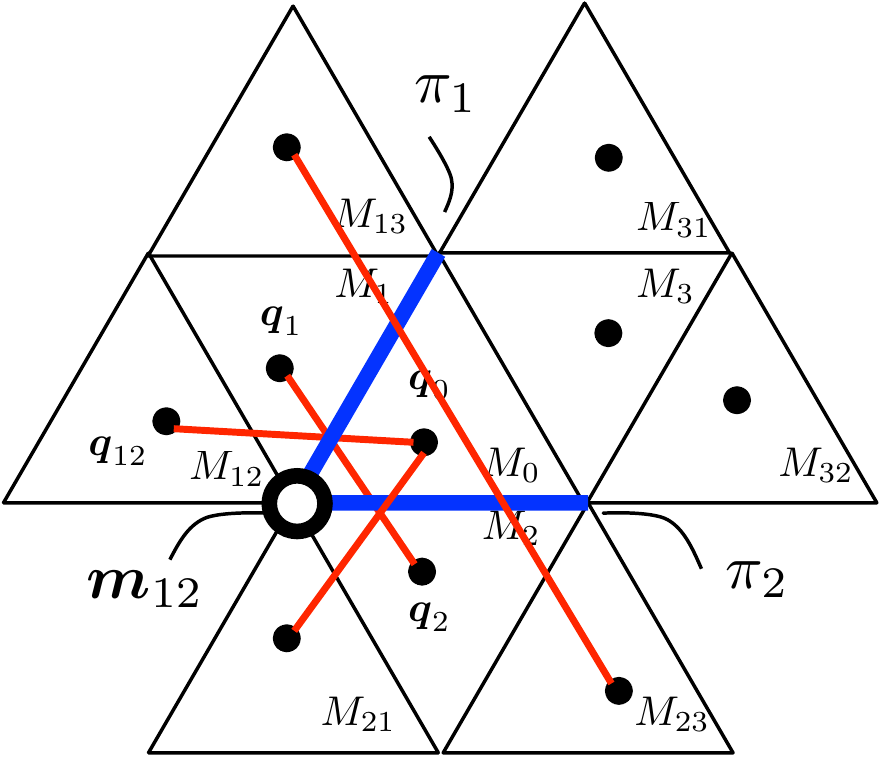}
\caption{Corresponding points for the orthogonality constraint\cite{takahashi12new}. Four pairs 
    $\langle\VEC{p}_{1}, \VEC{p}_{2}\rangle$, 
    $\langle\VEC{p}_{0}, \VEC{p}_{21}\rangle$,
    $\langle\VEC{p}_{12}, \VEC{p}_{0}\rangle$, and 
    $\langle\VEC{p}_{13}, \VEC{p}_{23}\rangle$ are available for the intersection $\VEC{m}_{12} = \VEC{n}_1 \times \VEC{n}_2$.}
\label{fig:corr_takahashi}
\end{figure}

{\bf[Baseline]} This baseline method is based on the simple bundle adjustment approach with the checkerboard as described in \cite{reshetouski11three}. Since the 3D geometry of the reference object is known, the 3D positions of the real image $\VEC{p}_0^{(l)}$ and their reflections $\VEC{p}_i^{(l)}$ and $\VEC{p}_{i,j}^{(l)}$ can be estimated by solving perspective-n-point (PnP)\cite{lepetit08epnp}. Here the superscript ${}^{(l)}$ indicates the $l$th landmark in the reference object. Once $N_l$ such landmark 3D positions are given, then the mirror normals can be computed simply by
\begin{equation}
  \begin{split}
    \VEC{n}_1 = \sum_l^{N_l} \VEC{l}_{1,2,3}^{(l)} / \left\|\sum_l^{N_l} \VEC{l}_{1,2,3}^{(l)}\right\|, \\
    \VEC{n}_2 = \sum_l^{N_l} \VEC{l}_{2,3,1}^{(l)} / \left\|\sum_l^{N_l} \VEC{l}_{2,3,1}^{(l)}\right\|, \\
    \VEC{n}_3 = \sum_l^{N_l} \VEC{l}_{3,1,2}^{(l)} / \left\|\sum_l^{N_l} \VEC{l}_{3,1,2}^{(l)}\right\|, \\
  \end{split}
\end{equation}
where $\VEC{l}_{i,j,k}^{(l)} = \VEC{p}_{i}^{(l)} - \VEC{p}_{0}^{(l)} + \VEC{p}_{ij}^{(l)} - \VEC{p}_{j}^{(l)} + \VEC{p}_{ik}^{(l)} - \VEC{p}_{k}^{(l)}$, and then the mirror distances can be computed by
\begin{equation}
  \begin{split}
    d_1 = \frac{1}{6{N_l}}\VEC{n}_1^\top\sum_l^{N_l}\left( \sum_{i=0}^3 \left( \VEC{p}_{i}^{(l)} \right) + \VEC{p}_{12}^{(l)} + \VEC{p}_{13}^{(l)}\right),\\
    d_2 = \frac{1}{6{N_l}}\VEC{n}_2^\top\sum_l^{N_l}\left( \sum_{i=0}^3 \left( \VEC{p}_{i}^{(l)} \right) + \VEC{p}_{23}^{(l)} + \VEC{p}_{21}^{(l)}\right),\\
    d_3 = \frac{1}{6{N_l}}\VEC{n}_3^\top\sum_l^{N_l}\left( \sum_{i=0}^3 \left( \VEC{p}_{i}^{(l)} \right) + \VEC{p}_{31}^{(l)} + \VEC{p}_{32}^{(l)}\right).\\
  \end{split}
\end{equation}
Note that the above PnP procedure requires a non-linear reprojection error minimization process in practice.

\begin{figure*}[t]
    \centering
    \includegraphics[width=1\linewidth]{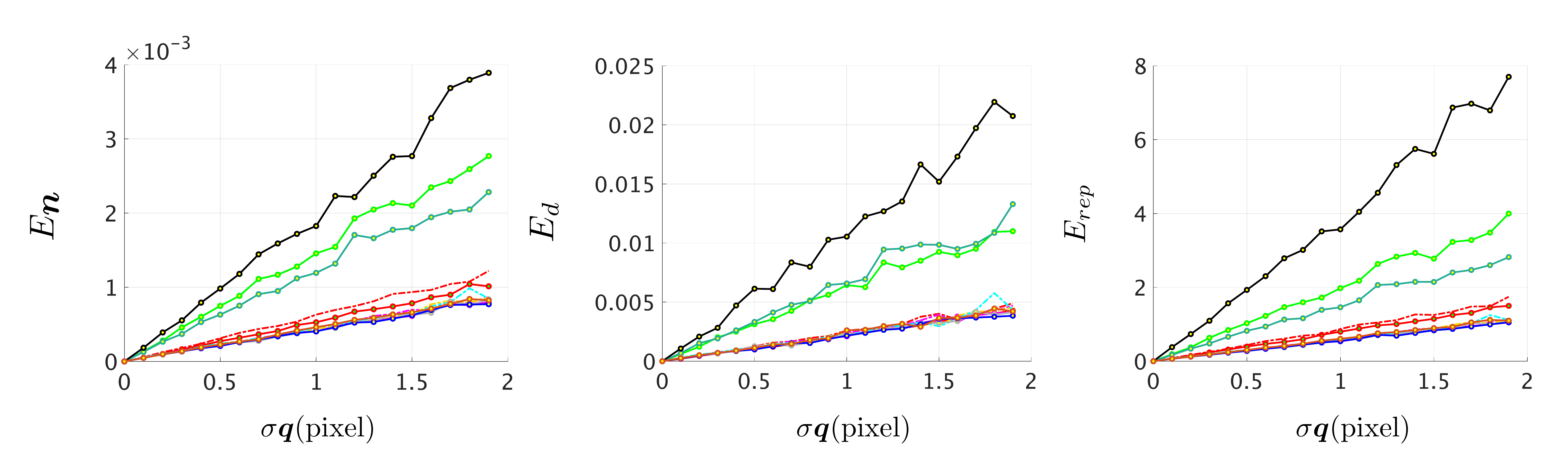}
    \caption{Estimation errors at different noise levels $\sigma_{\VEC{q}}$. Since the reprojection errors $E_{rep}$ after bundle adjustment are close to each other, $E_{rep}$ at $\sigma_{\VEC{q}}=1.0$ are also reported in Fig.~\ref{fig:label} with legends.}
    \label{fig:simulation_noise}
\end{figure*}

\begin{figure*}[t]
  \centering
  \includegraphics[width=1\linewidth]{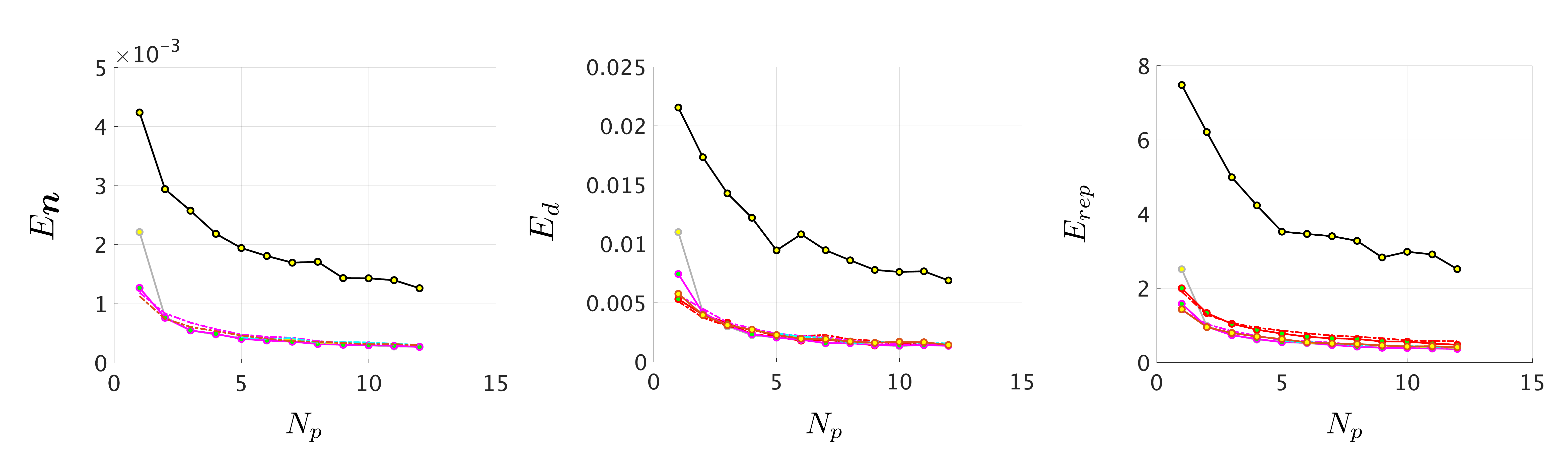}
  \caption{Estimation errors at different numbers of reference points $N_p$. Legends are provided in Figure \ref{fig:label}. The results after bundle adjustment are almost the same.}
  \label{fig:simulation_point}
\end{figure*}

{\bf[Absolute conic based approach]} Ying \etal\cite{ying10geometric,ying13self} provided some important insights on mirror reflections, such as the multi-reflections of a 3D point lie in a 3D circle $\Omega_{3D}$ that is imaged into a conic $\Omega$. The circular points can be determined by the intersection of this projected conic and the absolute conic computed from the intrinsic parameters. These circular points provide a rectification matrix that transform the projected conic to a 3D circle $\Omega'_{3D}$. By solving the conventional PnP problem between the 3D points on $\Omega'_{3D}$ and the 2D points on $\Omega$, the coordinates of the 3D points on $\Omega'_{3D}$ in camera coordinate system can be computed. On the basis of this procedure, for example, the mirror normal $\VEC{n}_1$ and $\VEC{n}_2$ can be estimated from mirror reflections of a single 3D point consisting of a 3D circle, $\VEC{p}_0, \VEC{p}_{1}, \VEC{p}_{2}, \VEC{p}_{12}, \VEC{p}_{21}$, as follows,
\begin{equation}
    \begin{split}   
        \VEC{n}_0 &= (\VEC{p}_1 - \VEC{p}_0) / ||\VEC{p}_1 - \VEC{p}_0||, \\
        \VEC{n}_1 &= (\VEC{p}_2 - \VEC{p}_0) / ||\VEC{p}_2 - \VEC{p}_0||.
    \end{split}
\end{equation}
In the case of $N_m = 3$, the two candidates of each mirror parameter can be obtained, for example $\VEC{n}_1$ from two point sets ($\VEC{p}_0, \VEC{p}_1, \VEC{p}_{2}, \VEC{p}_{12}, \VEC{p}_{21}$) and ($\VEC{p}_0, \VEC{p}_1, \VEC{p}_3, \VEC{p}_{13}, \VEC{p}_{31}$). In addition, the $N_p$ 3D points produce $N_p$ variations of the mirror parameters on the basis of this procedure. Here, we utilized the average of each parameter as the output of this approach in these evaluations.

{\bf[Orthogonality constraint based approach]}
As pointed out by Takahashi\etal\cite{takahashi12new}, two 3D points $\VEC{p}_i$ and $\VEC{p}_j$ defined as reflections of a 3D point by different mirrors of normal $\VEC{n}_i$ and $\VEC{n}_j$ respectively satisfy an orthogonality constraint:
\begin{equation}
  \left( \VEC{p}_i - \VEC{p}_j \right)^\top \left( \VEC{n}_i \times \VEC{n}_j \right) = \left( \VEC{p}_i - \VEC{p}_j \right)^\top \VEC{m}_{ij} = 0. \label{eq:orthogonality}
\end{equation}
As illustrated by Figure \ref{fig:corr_takahashi}, this constraint on $\VEC{m}_{12}$ holds for four pairs 
$\langle\VEC{p}_{1}, \VEC{p}_{2}\rangle$, 
$\langle\VEC{p}_{0}, \VEC{p}_{21}\rangle$,
$\langle\VEC{p}_{12}, \VEC{p}_{0}\rangle$, and 
$\langle\VEC{p}_{13}, \VEC{p}_{23}\rangle$ 
as the reflections of $\VEC{p}_0$, $\VEC{p}_1$, $\VEC{p}_2$, and $\VEC{p}_3$, respectively. Similarly, 
$\langle\VEC{p}_{2}, \VEC{p}_{3}\rangle$, 
$\langle\VEC{p}_{21}, \VEC{p}_{31}\rangle$, 
$\langle\VEC{p}_{0}, \VEC{p}_{32}\rangle$, and
$\langle\VEC{p}_{23}, \VEC{p}_{0}\rangle$
can be used for computing $\VEC{m}_{23} = \VEC{n}_2 \times \VEC{n}_3$, and
$\langle\VEC{p}_{3}, \VEC{p}_{1}\rangle$,
$\langle\VEC{p}_{31}, \VEC{p}_{0}\rangle$,
$\langle\VEC{p}_{32}, \VEC{p}_{12}\rangle$, and
$\langle\VEC{p}_{0}, \VEC{p}_{13}\rangle$
can be used for $\VEC{m}_{31} = \VEC{n}_3 \times \VEC{n}_1$. Once the intersection vectors $\VEC{m}_{12}$, $\VEC{m}_{23}$, and $\VEC{m}_{31}$ are obtained, the mirror normals and the distances can be estimated linearly as described in Takahashi\etal \cite{takahashi12new}.

In addition to the above methods, we also compared the proposed method with its earlier version\cite{takahashi17linear} in which the 3D points were restricted to be on the line through its 2D observations in the base chamber, while the proposed method does not have this limitation as described in Section \ref{sec:ba}.

\begin{figure}[t]
  \centering
  \includegraphics[width=1\linewidth]{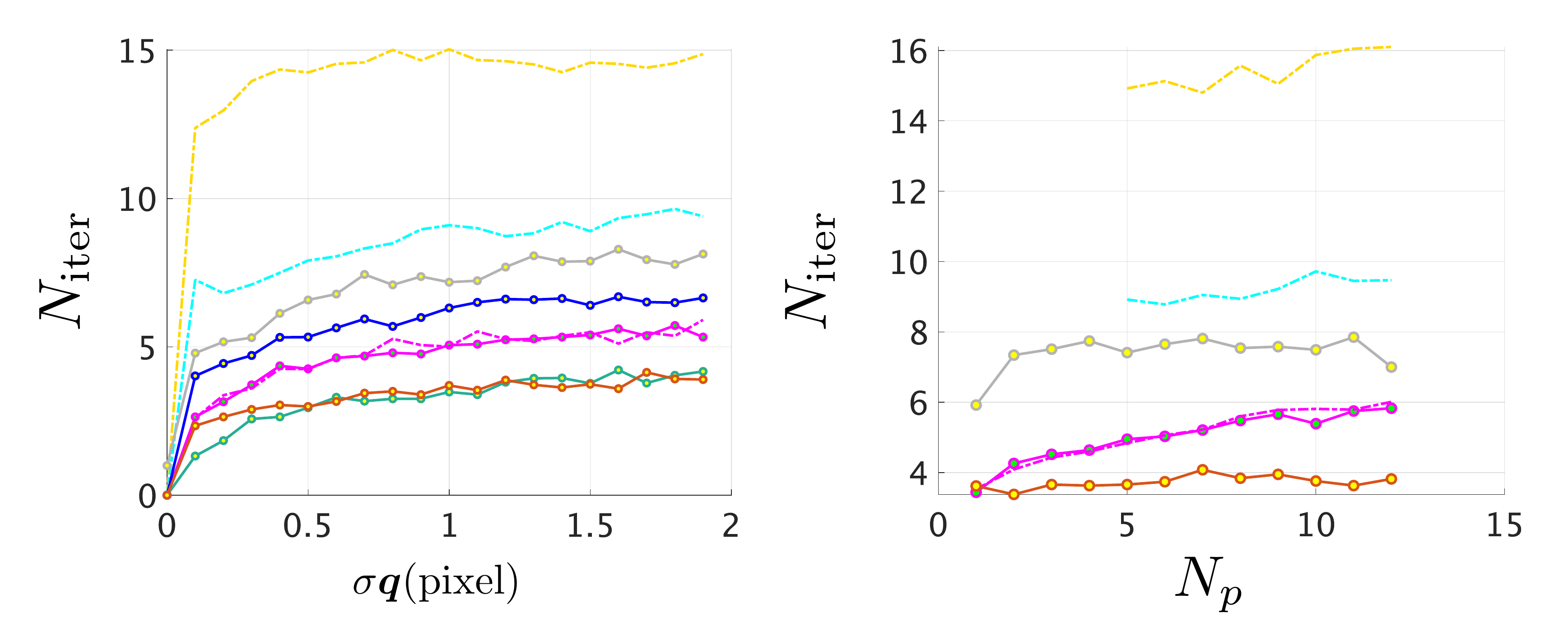}
  \caption{Number of iterations until convergence at different $\sigma_{\VEC{q}}$ with $N_p=5$ (left) and at different $N_p$ with $\sigma_{\VEC{q}}=1$ (right). Legends are provided in Figure \ref{fig:label}.}
  \label{fig:simulation_iteration}
\end{figure}

\begin{figure}[t]
  \centering    
  \centering
  \includegraphics[width=\linewidth]{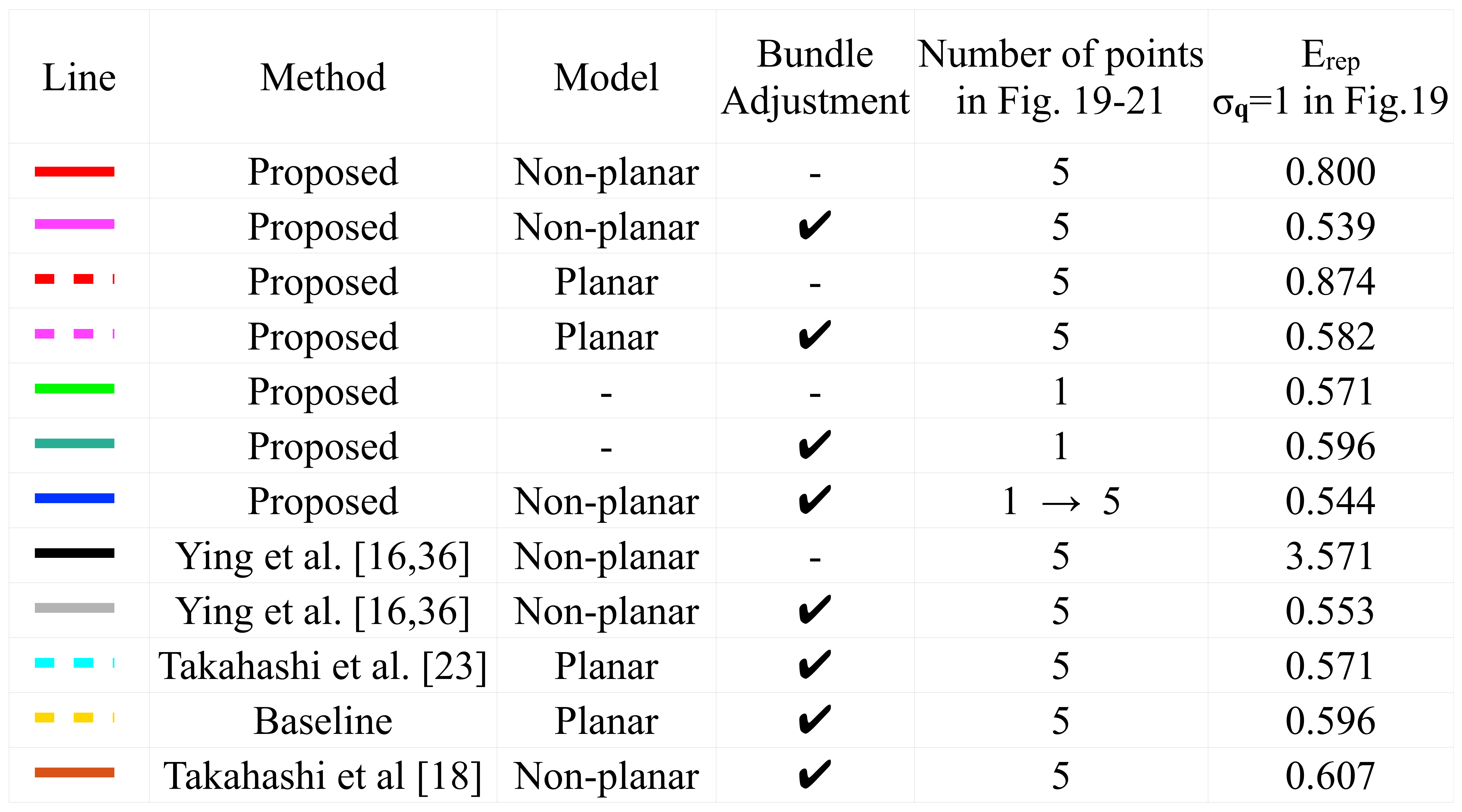}
  \caption{Legends and $E_{rep}$ at $\sigma_{\VEC{q}}=1$ in Figure \ref{fig:simulation_noise} for each configurations.}
  \label{fig:label}
\end{figure}

The following three error metrics are used in this section to evaluate the performance of the proposed method in comparison with the above-mentioned conventional approaches quantitatively. The average estimation error of normal $E_{\VEC{n}}$ measures the average angular difference from the ground truth by
\begin{equation}
    E_{\VEC{n}} = \frac{1}{3} \sum_{i=1}^3 \left|\cos^{-1}(\VEC{n}_i^\top \check{\VEC{n}}_{i}) \right|,
\end{equation}
where $\check{\VEC{n}}_i \: (i=1,2,3)$ denotes the ground truth of the normal $\VEC{n}_{i}$.  The average estimation error of distance $E_{d}$ is defined as the average $L_1$-norm to the ground truth:
\begin{equation}
    E_{d} = \frac{1}{3} \sum_{i=1}^{3} |d_i - \check{d}_{i}|,
\end{equation}
where $\check{d}_{i} \: (i=1,2,3)$ denotes the ground truth of the distance $d_{i}$.  Also, the average reprojection error $E_\mathrm{rep}$ is defined as:
\begin{equation}
    E_\mathrm{rep} = \frac{1}{10{N_l}} \sum_{l=1}^{{N_l}} \left| \VEC{E}^{(l)}(\VEC{p}_0, \VEC{n}_1, \VEC{n}_2, \VEC{n}_3, d_1, d_2, d_3) \right|,
\end{equation}
where $\VEC{E}^{(l)}(\cdot)$ denotes the reprojection error $\VEC{E}(\cdot)$ defined by Eq. \eqref{eq:reproj_error_vec} at $l$th point.

\subsubsection{Quantitative evaluations with synthesized data}
This section provides a quantitative performance evaluation using a synthesized dataset. A virtual camera and three mirrors are arranged in accordance with the real setup (Figure \ref{fig:setup}). By virtually capturing 3D points simulating a reference object, the corresponding 2D kaleidoscopic projections used as the ground truth are generated first, and then random pixel noise is injected to them at each trial of calibration.

Figures \ref{fig:simulation_noise} and \ref{fig:simulation_point} report average estimation errors $E_{\VEC{n}}$, $E_d$, $E_\mathrm{rep}$ over 100 trials at different noise levels and different numbers of reference points. In these figures $\sigma_{\VEC{q}}$ denotes the standard deviation of zero-mean Gaussian pixel noise and $N_p$ denotes the number of 3D points used in the calibration. In Figure \ref{fig:simulation_iteration}, $N_\mathrm{iter}$ denotes the number of iterations required by the kaleidoscopic bundle adjustment in each evaluation. Figure \ref{fig:label} shows the legends of lines in Figures \ref{fig:simulation_noise} - \ref{fig:simulation_iteration} and reprojection errors $E_{rep}$ of each method at $\sigma_{\VEC{q}} = 1$.

\begin{figure}[t]
  \centering
    \includegraphics[width=\linewidth]{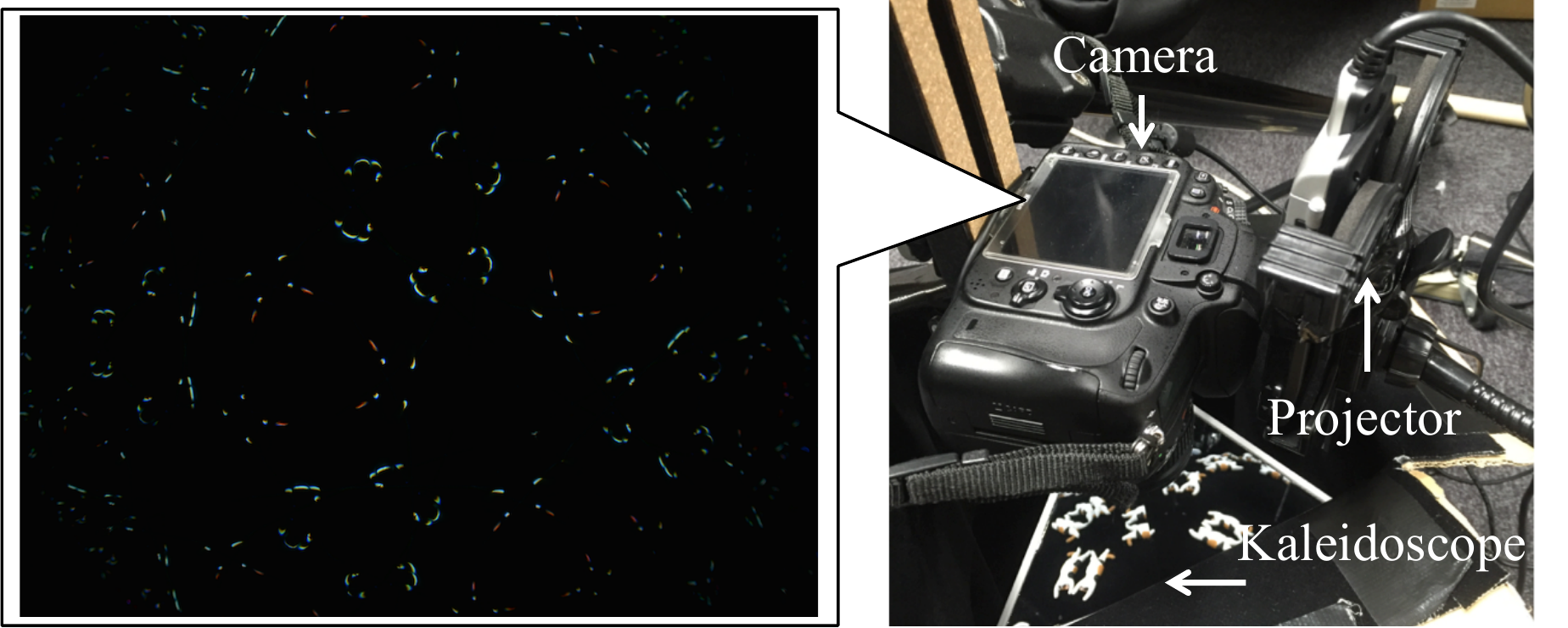}
    \caption{A kaleidoscopic capture setup for 3D reconstruction. It consists of three first surface mirrors, a camera, and a laser projector.}
    \label{fig:setup}
\end{figure}

\begin{figure}[t]
  \centering
    \includegraphics[width=0.6\linewidth]{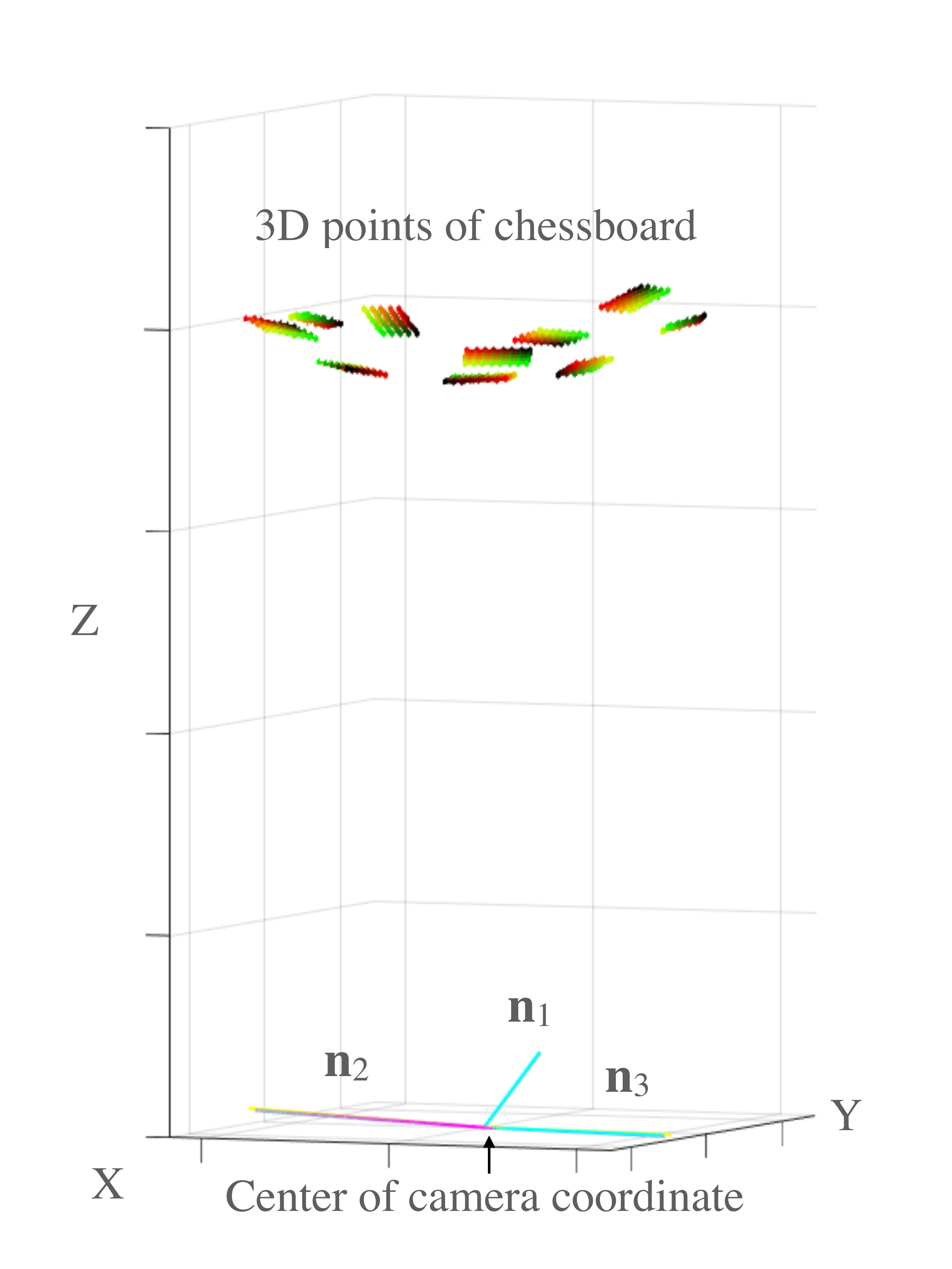}
    \caption{Calibration results. The colored lines in the bottom illustrate $d \VEC{n}$ (\ie, the foot of perpendicular from the camera center) of each mirror. Here magenta, gray, cyan, yellow lines represents the results by proposed method, ying \etal\cite{ying10geometric,ying13self}, takahashi\etal\cite{takahashi12new}, and baseline respectively. Their results were almost the same. The 10 patterns in the top illustrate the 3D points estimated by PnP.}
    \label{fig:real_chess}
\end{figure}

\begin{figure}[t]
  \centering
    \includegraphics[width=0.7\linewidth]{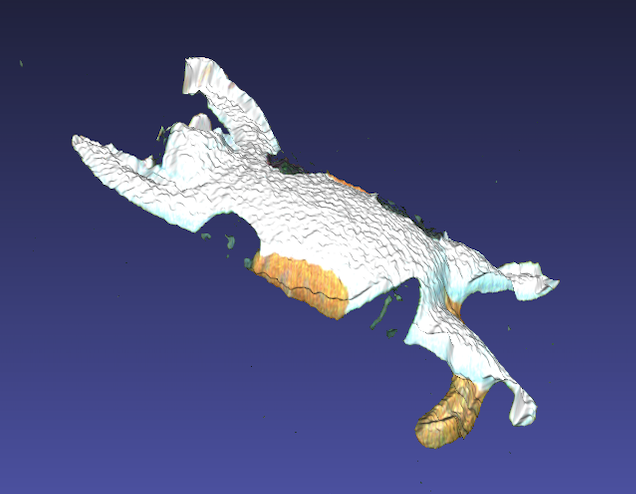}
    \caption{Reconstructed 3D shape of {\it cat} object.}
    \label{fig:ststereo}
\end{figure}

As shown in Figure \ref{fig:label}, the magenta and red lines denote the results by the proposed method with and without the kaleidoscopic bundle adjustment (Section \ref{sec:ba}). The proposed method uses kaleidoscopic projections of five random non-planar 3D points, while the dashed red and magenta lines are the results with planar five points simulating the chessboard (Figure \ref{fig:chessboard}). The light and dark green lines are the results with a single 3D point generated randomly followed by the kaleidoscopic bundle adjustment or not. Besides, the blue line is the result of the kaleidoscopic bundle adjustment with five points to the linear solutions from a single point.

The black line is the result of the absolute conic based approach \cite{ying10geometric,ying13self} with five input points, and the gray line is the result by applying our kaleidoscopic bundle adjustment to it. The yellow and cyan dashed lines are the results by the orthogonality constraint-based approach \cite{takahashi12new} and the baseline with the same five points for the red and magenta dashed lines. The orange line shows the result of the earlier version of this work \cite{takahashi17linear} with five input points. Note that the baseline and the orthogonality constraint-based approach \cite{takahashi12new} without the final non-linear optimization could not achieve comparable results (typically $E_\mathrm{rep} \gg 10$ pixels). Also, these methods using 3D reference positions without applying non-linear refinement after a linear PnP\cite{lepetit08epnp} could not estimate valid initial parameters for the final non-linear optimization. Therefore, they are omitted in these figures.

From both Figs. \ref{fig:simulation_noise} and \ref{fig:simulation_point}, we can observe the linear solutions by the proposed method outperform those by the absolute conic based approach\cite{ying10geometric,ying13self}. This is because each mirror parameter is implicitly constrained by the other mirrors in the proposed method, such as in Eqs. \eqref{eq:n1} and \eqref{eq:n2}, whereas they are estimated on an apparently per-mirror basis in the absolute conic based approach. This improvement of linear solutions can contribute to reduce the computational cost as shown in Figure \ref{fig:simulation_iteration}. Additionally, increasing the number of reference points improves the estimation accuracy of all methods as shown in Figure \ref{fig:simulation_point}.

 These figures also show that the optimized results are close to each other, but as reported in Figure \ref{fig:label}, the proposed method with non-planar points and kaleidoscopic bundle adjustment approach outperforms the other methods in terms of reprojection error. 


Besides, the non-planar based methods are slightly better than the planar based method in Figure \ref{fig:simulation_point}. This is because the randomly scattered non-planar points in a 3D space give stronger constraints for 3D analysis than the points on a single plane. 

Whereas the Baseline method and Orthogonality constraint-based method require a reference object of known geometry, such as a chessboard, the proposed method does not. Using this feature, the proposed method can simultaneously estimate the camera parameters and 3D structure of a 3D object.

From these results, we can conclude that the proposed method is precise, effective, and practical.

\subsection{Qualitative evaluations with real data}
Figure \ref{fig:setup} shows our kaleidoscopic capture setup. The intrinsic parameter $A$ of the camera (Nikon D600, $6016{\times}4016$ resolution) is calibrated beforehand\cite{zhang2000flexible}, and the camera observes the target object \textit{cat} (about $4\times5\times1$ cm) with three planar first surface mirrors. The projector (MicroVision SHOWWX+ Laser Pico Projector, $848{\times}480$ resolution) is used to cast line patterns to the object for simplifying the correspondence search problem in a light-sectioning fashion (Figure \ref{fig:setup} left), and the projector itself is not involved in the calibration \wrt the camera and the mirrors.


\begin{table}[tb]
  \centering
  \caption{The reprojection errors of linear solutions and optimized solutions by each method. The $E_{rep}$ (all) column reports the average reprojection error of all the points. The 0th, 1st, and 2nd ref columns report the average error of the direct observations, that of the first and second reflections respectively.}\label{tb:real_data_reps}
  \begin{tabular}{ccccc}
     Method & $E_{rep}$ (all) & 0th ref & 1st ref & 2nd ref\\ \hline\hline    
     Proposed (linear) & 5.49 & 2.84 & 3.42 & 6.97\\ 
     Proposed (BA) & 3.85 & 4.24 & 3.55 & 3.94\\ 
     Ying \etal \cite{ying10geometric,ying13self} (linear) & 317.91 & 273.31 & 302.38 & 333.11 \\ 
     Ying \etal \cite{ying10geometric,ying13self} (BA) & 3.88 & 4.18 & 3.61 & 3.96\\ 
     Takahashi \etal \cite{takahashi12new} (linear) & 433.43 & 330.52 & 420.98 & 456.80\\
     Takahashi \etal \cite{takahashi12new} (BA) & 3.88 & 3.83 & 3.67 & 4.00\\ 
     Baseline (linear) &  618.21 & 351.60& 495.74 & 723.88\\ 
     Baseline (BA) & 3.84 & 3.88 & 3.98 & 3.76\\ 
  \end{tabular}
\end{table}

Figure \ref{fig:chessboard} shows a captured image of a chessboard, and Figure \ref{fig:real_chess} shows the mirror normals and distances calibrated by the proposed and conventional methods. In addition, Table \ref{tb:real_data_reps} reports the reprojection errors of linear solutions and optimized solutions by each method. The results show that the performance of each method after applying the kaleidoscopic bundle adjustment has the same tendency reported in the simulation results. The optimized reprojection errors are higher than those in the simulation results. This is because of the localization accuracy of corresponding points and nonplanarity of mirrors. 

Figure \ref{fig:ststereo} shows a 3D rendering of the estimated 3D shape using the mirror parameters calibrated by the proposed method, while the residual reprojection error indicates the parameters can be further improved for example through the 3D shape reconstruction process itself\cite{furukawa2009accurate}.

From these results, we can conclude that the proposed method performs reasonably and provides sufficiently accurate calibration for 3D shape reconstruction.

\section{Discussion}\label{sec:section_8}
\subsection{Ambiguity of chamber assignment}
In the case of the 8-point algorithm for the regular two-view extrinsic calibration\cite{hartley00multiple}, the linear algorithm returns four possible combinations of the rotation and the translation, and we can choose the right combination by examining if triangulated 3D points appear in front of the cameras. The mirror normal estimation in Section \ref{sec:base_structure} is a special case of the 8-point algorithm, and this has such sign ambiguity on the mirror normal as described in Section \ref{sec:base_structure}. This ambiguity is also solved by considering the result of triangulation. In other words, estimating the essential matrix is identical to estimating the mirror normal.

In addition to the sign ambiguity, the normal estimation for the kaleidoscopic system has another family of ambiguity due to multiple reflections. As introduced in Section \ref{sec:base_structure}, particular combinations of kaleidoscopic projections can return physically infeasible solutions, and they can be rejected by additional geometric constraints as done for the 8-point algorithm. However, there exists another class of solutions due to a \textit{sparse sampling} of the observations.

Consider a base structure by the pairs $\PAIR{\VEC{q}_0}{\VEC{q}_{12}}$ and $\PAIR{\VEC{q}_2}{\VEC{q}_{121}}$ in Figure \ref{fig:kaleidoscopic_image}. This configuration can estimate the mirror parameters successfully, one between $\VEC{p}_0$ and $\VEC{p}_2$, and the other between $\VEC{p}_0$ and $\VEC{p}_{12}$. Although the latter is a virtual mirror, this interpretation satisfies all the constraints in Section \ref{sec:base_structure}. In other words, we can assemble a mirror system of this configuration in practice.

To solve this problem, Section \ref{sec:assignment} utilizes the recall ratio (Eq. \eqref{eq:recall}) so that our algorithm returns the solution that reproduces as many as possible candidates points $\VEC{r}$ observed in the image.

\subsection{Approximation in label propagation}\label{sec:discussion_chamber_assignment} 
As described in Section \ref{sec:section_5}, we approximated the original chamber assignment algorithm by the nearest neighbor search. The evaluations experimentally prove that this approximation can substitute the original algorithm since the proposed algorithm worked correctly in most of the cases, particularly for the direct and the first reflections.

For multiple reflections with noisy inputs, the assignment accuracy is degraded. This can be attributed to the accuracy of the mirror normal estimation itself being degraded and the projections of multiple reflections being able to be synthesized with large reprojection errors.

\subsection{RANSAC or PROSAC approach for chamber assignment}\label{sec:discussion_ransac}
Although the proposed algorithm examines all possible base structures as introduced in Algorithm \ref{alg:overview} to evaluate the performance thoroughly, we can also consider a RANSAC (random sample consensus) or PROSAC (progressive sample consensus) approach\cite{chum2005matching}. For example, we can first hypothesize the base chamber from $N_r$ candidates and then can consider only $N_{\pi} + 1$ nearest points around it for estimating the mirror parameters. The other idea is to hypothesize the base chamber and to add some geometric procedures such as flipping the point pairs for finding the correct labeling efficiently. Designing and evaluating such an approach is one of our future works.

\subsection{Pruning of base structure candidates}\label{sec:discussion_pruning}
As described in Section \ref{sec:section_5}, our chamber assignment algorithm firstly lists all candidates of labeling of the base structure and prunes inappropriate candidates by using three geometric constraints, \ie, Eq. \eqref{eq:pruning}, Proposition 1, and Proposition 2. This section demonstrates how many candidates are removed by these constraints experimentally.

\begin{table}[tb]
  \centering
  \caption{The number and percentage of candidates that passed each pruning constraint in case of $\sigma_{\VEC{q}}=0$ and $\sigma_{\VEC{q}}=2$. Note that the total number of candidates is $151,200$.}\label{tb:pruning_noise_0}
  \begin{tabular}{ccc}
     Constraints & $\sigma_{\VEC{q}} = 0$ & $\sigma_{\VEC{q}} = 2$\\ \hline\hline    
     Eq. \eqref{eq:pruning} & 3,552 (2.33 \%) & 4,608 (3.05 \%)\\ 
     Proposition 1 & 25,200 (16.67 \%) & 25,200 (16.67 \%)\\ 
     Proposition 2 & 3,386 (2.24 \%)& 2,996 (1.98 \%)\\
     Proposition 1 and 2 & 796 (0.53 \%)& 492 (0.33 \%)\\
     All & 36 (0.023 \%) & 54 (0.036 \%)\\
  \end{tabular}
\end{table}

Here, we assume a case of three mirrors and second reflection as in Figure \ref{fig:base_structure} (b). Since the base structure consists of three doublets that consist of six points, the number of candidates of labeling is ${}_{10}\mathrm{P}_6 = 152100$ for six appropriate labeling candidates. 

Table \ref{tb:pruning_noise_0} reports the results of pruning in case of noise $\sigma_{\VEC{q}} = 0$ and $\sigma_{\VEC{q}} = 2$. These results show that using each constraint contributes to remove many of the labeling candidates. Whereas the performance of using Eq. \eqref{eq:pruning} depends on an experimentally defined threshold that should be almost zero, Proposition 1 and 2 work with clearer conditions. We can see that using both these propositions (``Proposition 1 and 2'') and using these all geometric constraints (``All'') outperform the results of using each constraint. They show that all constraints are complementary. Note that we found that the appropriate labeling candidates were included in the filtered results in all cases. 

These results show that using the proposed pruning rules can reduce the computational cost significantly.

\subsection{Degenerate cases}
Both proposed algorithms of chamber assignment and mirror parameter estimation are based on the kaleidoscopic projection constraint (Eq. \eqref{eq:kaleidoscopic_projection_constraint}) satisfied by second or further reflections. Therefore, these algorithms do not work in two cases. (1) If the two mirror are parallel, the mirror normals are not computable by solving Eq. \eqref{eq:kpc_use}, Eq. \eqref{eq:n1}, and Eq. \eqref{eq:n2} because the constraints are linearly dependent. (2) If the second reflections are not observable due to the angle of view or discontinuities, the mirror normals are not computable. Especially, in case of using more than three mirrors, discontinuities are more likely to happen in general, and the second reflections themselves become difficult to find (Figure \ref{fig:various_configurations}).

\begin{figure}[t]
\centering
\includegraphics[width=1\linewidth]{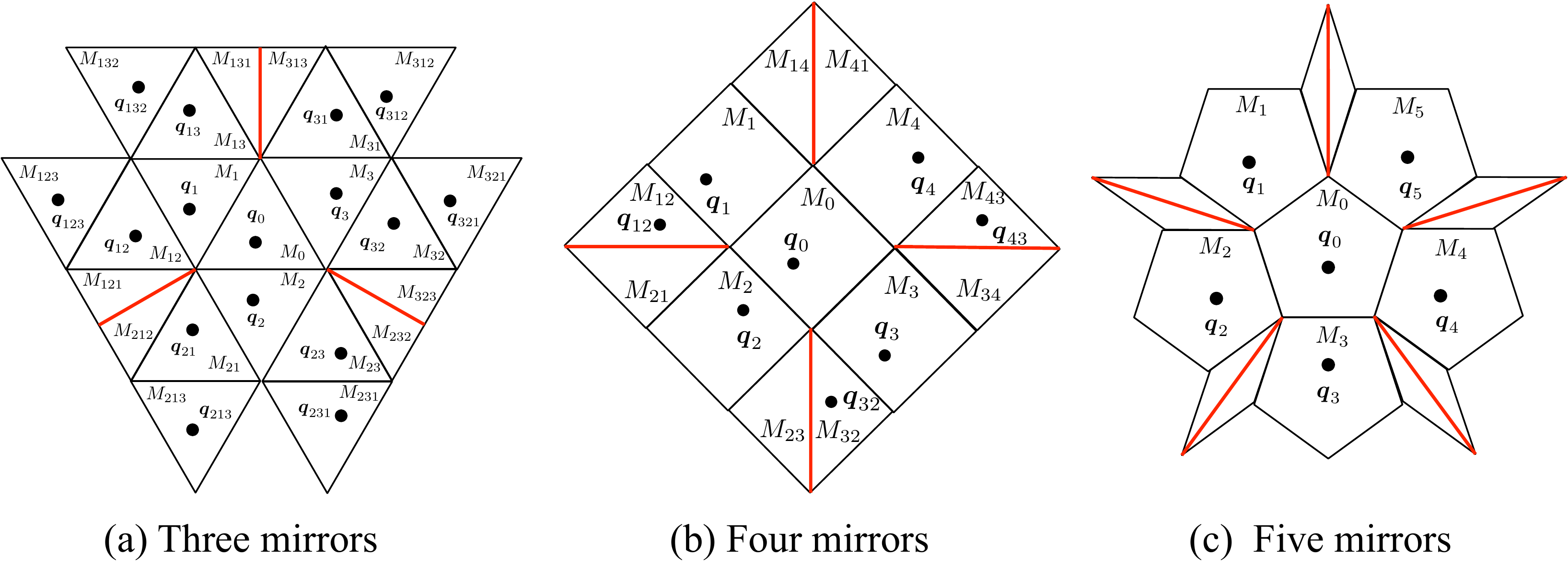}
\caption{Kaleidoscopic imaging system using (a) three, (b) four, and (c) five mirrors. Discontinuities (red lines) appear on the boundaries of overlapping chambers.}
\label{fig:various_configurations}
\end{figure}

\section{Conclusion}\label{sec:section_9}
This paper proposes a novel algorithm of discovering the structure of a kaleidoscopic imaging system that consists of multiple planar mirrors and a camera. The key idea of our approach is to introduce a kaleidoscopic projection constraint, that is an epipolar constraint on a mirror plane that is satisfied by projections of high-order reflections. This constraint enables the proposed method to conduct a 3D validation on candidates of chamber assignments and to derive mirror parameters linearly. Evaluations with synthesized data and real data prove that the proposed method discovers the structure of a kaleidoscopic imaging system in various configurations, \eg, the mirror is not orthogonal to a common ground.

As discussed in Section \ref{sec:section_8}, our future work includes developing an efficient algorithm on the basis of RANSAC / PROSAC as well as integrating the intrinsic parameter estimation towards fully-automatic self-calibration of kaleidoscopic imaging system.

\ifCLASSOPTIONcompsoc
  \section*{Acknowledgments}
  This research is partially supported by JSPS Kakenhi Grant Number 26240023 and 18K19815.
\else
  \section*{Acknowledgment}

\fi

\ifCLASSOPTIONcaptionsoff
  \newpage
\fi

\bibliographystyle{IEEEtran}
\bibliography{pami2017}




\begin{IEEEbiography}[{\includegraphics[width=1in,height=1.25in,clip,keepaspectratio]{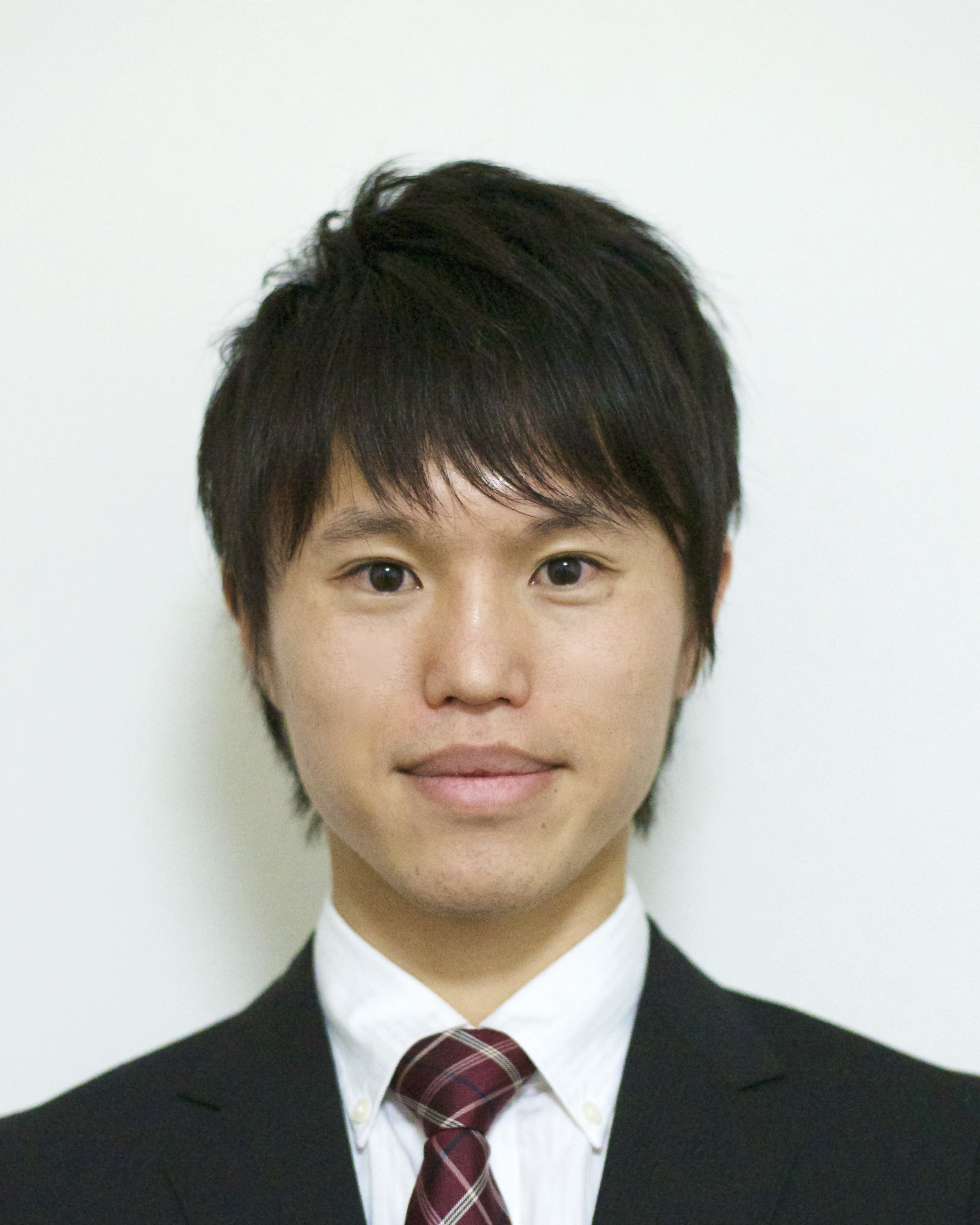}}]{Kosuke Takahashi}
 received his B.Sc. degree in engineering and M.Sc. and Ph.D. in informatics from Kyoto
 University, Japan, in 2010, 2012 and 2018, respectively. He is currently a research engineer at UMITRON. His research interest includes computer vision and its applications for sports performance enhancement and aquaculture. He received "Best Open Source
 Code" award Second Prize in CVPR 2012. He is a member of IPSJ.
 \end{IEEEbiography}

\begin{IEEEbiography}[{\includegraphics[width=1in,height=1.25in,clip,keepaspectratio]{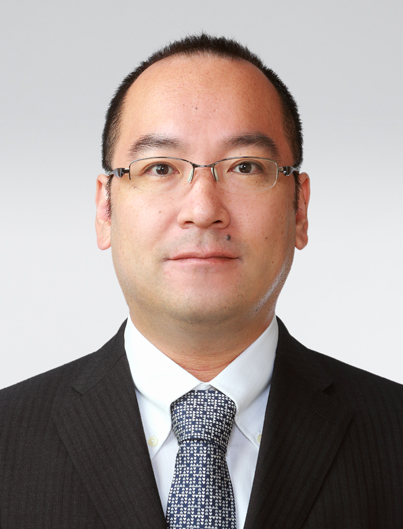}}]{Shohei Nobuhara}
received his B.Sc. in Engeneering, M.Sc. and Ph.D. in Informatics from Kyoto
University, Japan, in 2000, 2002, and 2005 respectively. Since 2019, he has been
an associate professor at Kyoto University. His research interest includes computer
vision and human shape and motion modelling. He is a member of IEEE, IPSJ, IEICE.
\end{IEEEbiography}

\end{document}